%% file: OLA.tex
\documentclass[journal,11pt,onecolumn,draftclsnofoot,]{IEEEtran}

\usepackage{graphicx}
\usepackage{amssymb}
\usepackage{amsbsy}
\usepackage{amsmath}
\usepackage{amsthm}
\usepackage{amssymb}
\usepackage{bbm}

\usepackage{cite}
\usepackage{stfloats}
\usepackage{url}

\usepackage{algorithm}
\usepackage{algorithmic}

\usepackage{enumerate}
\usepackage{color}
\usepackage{placeins}
\usepackage{float}
\usepackage{tabularx,colortbl}
\usepackage{times,amsmath,epsfig}
\usepackage{xspace,latexsym,syntonly}
\usepackage{amssymb}
\usepackage{amsfonts}
\usepackage{textcomp}
\usepackage{caption}
\usepackage{subcaption}
\usepackage{amsbsy}
\usepackage{cite}
\usepackage{mathtools, cuted}
\usepackage{lipsum, color}
\usepackage{epstopdf}

\newcommand{\1}{\mathbbm{1}}
\newcommand{\cZ}{\mathcal{Z}}

\newcommand{\DIS}{\mathcal{D}}

\input{macros}

\begin{document}\title{Disagreement-based Active Learning in \newline Online Settings}
\author{Boshuang Huang, Sudeep Salgia, Qing Zhao
\thanks{B.Huang, S. Salgia, Q. Zhao are with the School of Electrical and Computer Engineering, Cornell University, Ithaca, NY, 14850, USA. Emails: \{bh467,ss3827,qz16\}@cornell.edu.}\\
\thanks{This work was supported by the National Science Foundation under Grant CCF-1815559.}
}
\date{}
\maketitle
\begin{abstract}
\label{sec:abstract}
We study online active learning for classifying streaming instances within the framework of statistical learning theory. At each time, the learner 
either queries the label of the current instance or predicts the label based on past seen examples. The objective is to minimize the number of queries while constraining the number of prediction errors over a horizon of length $T$. We develop a disagreement-based online learning algorithm for a general hypothesis space and under the Tsybakov noise. We show that the proposed algorithm has a label complexity of $O(dT^{\frac{2-2\alpha}{2-\alpha}}\log^2 T)$ under a constraint of bounded regret in terms of classification errors, where $d$ is the VC dimension of the hypothesis space and $\alpha$ is the Tsybakov noise parameter. We further establish a matching (up to a poly-logarithmic factor) lower bound, demonstrating the order optimality of the proposed algorithm. We address the tradeoff between label complexity and regret and show that the algorithm can be modified to operate at a different point on the tradeoff curve. 
%
\end{abstract}
%
%

\section{Introduction}
\label{sec:intro}

We consider online classification of streaming instances within the framework of statistical learning theory. Let $\{X_t\}_{t\ge 1}$ be a sequence of instances drawn independently at random from an unknown underlying distribution $\mathbb P_X$ over an instance space $\mathcal X$. Each instance $X_t$ has a hidden binary label $Y_t\in \{0, 1\}$ that relates probabilistically to the instance  according to an unknown conditional distribution $\mathbb P_{Y|X}$. The learner is characterized by its hypothesis space $\mathcal H$ consisting of all classifiers under consideration. At each time $t$, the learner decides whether to query the label of the current instance $X_t$. If yes, $Y_t$ is revealed. Otherwise, the learner predicts the label of $X_t$ using a hypothesis in $\mathcal H$ and incurs a classification error if the predicted label does not equal to the true label $Y_t$. The objective is to minimize the expected number of queries over a horizon of length $T$ while constraining the total number of classification errors. The tension between label complexity and classification error rate needs to be carefully balanced through a sequential strategy governing the query and labeling decisions at each time.

The above problem arises in applications such as spam detection and event detection in real-time surveillance.  The key characteristics of these applications are the high-volume streaming of instances and the complex and nuanced definition of labels. While the latter necessitates human intervention to provide annotations for selected instances, such human annotations, time consuming and expensive to obtain, should be sought after sparingly to ensure scalability. 

\subsection{Previous Work on Active Learning}

The above problem falls under the general framework of active learning. In contrast to passive learning where labeled examples are given \emph{a priori} or drawn at random, active learning asserts control over which labeled examples to learn from by actively querying the labels for carefully selected instances. The hope is that by learning from the most informative examples, the same level of classification accuracy can be achieved with much fewer labels than in passive learning.

\paragraph{Offline Active Learning} 

Active learning has been studied extensively under the Probably Approximately Correct (PAC) model, where the objective is to output an $\epsilon$-optimal classifier with probability $1-\delta$ using as few labels as possible. The PAC model pertains to offline learning since the decision maker does not need to self label any instances during the learning process. An equivalent view is that classification errors that might have incurred during the learning process are inconsequential, and the tension between label complexity and classification errors is absent. If measured purely by label complexity, the decision maker has the luxury of skipping, at no cost, as many instances as needed to wait for the most informative instance to emerge.

A much celebrated  active learning algorithm was given by Cohn, Atlas, and Ladner~\cite{cohn1994improving}. Named after its inventors, the CAL algorithm  is applicable to a general hypothesis space $\mathcal H$. It, however, relies on the strong assumption of realizability, i.e., the instances are perfectly separable and there exists an error-free classifier in $\mathcal H$. In this case, hypotheses inconsistent with a single label can be safely eliminated from further consideration. Based on this key fact, CAL operates by maintaining two sets at each time: the \emph{version space} consisting of all surviving hypotheses (i.e., those that are consistent with all past labels), and the \emph{region of disagreement} (RoD), a subset of $\mathcal X$ for which there is disagreement among hypotheses in the current version space regarding their labels. CAL queries labels if and only if the instance falls inside the current RoD. Each queried label reduces the version space, which in turn may shrink the RoD, and the algorithm iterates indefinitely. Note that instances outside the RoD are given the same label by all the hypotheses in the current version space. It is thus easy to see that CAL represents a conservative approach: it only disregards instances whose labels can already be perfectly inferred from past labels. Quite surprisingly, by merely avoiding querying labels that carry no additional information, exponential reduction in label complexity can be achieved in a broad class of problems. (See, for example, an excellent survey by Dasgupta~\cite{dasgupta2011two} and a monograph by Hanneke~\cite{hanneke2014theory}).

The CAL algorithm was extended to the agnostic setting by Balcan, Beygelzimer, and Langford~\cite{balcan2006agnostic}. In the agnostic setting, instances are not separable, and even the best classifier $h^*$ in $\mathcal H$ experiences a non-zero error rate. The main challenge in extending CAL to the agnostic case is the update of the version space: a single inconsistent label can no longer disqualify a hypothesis, and the algorithm needs to balance the desire of quickly shrinking the version space with the irreversible risk of eliminating $h^*$. Referred to as $A^2$ (Agnostic Active), the algorithm developed by Balcan, Beygelzimer, and Langford explicitly maintains an $\epsilon$ neighborhood of $h^*$ in the version space by examining the empirical errors of each hypothesis. Analysis of the $A^2$ algorithm can be found in~\cite{hanneke2007bound,hanneke2009adaptive,koltchinskii2010rademacher,hanneke2011rates}. Variants of the $A^2$ algorithm include \cite{beygelzimer2008importance,beygelzimer2010agnostic,beygelzimer2011efficient,dasgupta2008general, hanneke2012surrogate}.  In particular, the DHM algorithm (named after the authors) in~\cite{dasgupta2008general} simplifies the maintenance of the RoD through a reduction to supervised learning.

The above conservative approach originated from the CAL algorithm is referred to as the disagreement-based approach. The design methodology of this conservative approach focuses on avoiding querying labels that provide no or little additional information. More aggressive approaches that actively seeking out more informative labels to query have been considered in the literature. One such approach is the so-called margin-based. It is specialized for learning  homogeneous (i.e. through the origin) linear separators of  instances on the unit sphere in $\mathbb R^d$ and adopts  a specific noise model that assumes linearity in terms of the inner product with the Bayes optimal classifier. In this case, the informativeness of a potential label can be measured by how close the instance is to the current decision boundary. Representative work on the margin-based approach includes~\cite{dasgupta2005analysis,balcan2007margin,balcan2013active,awasthi2014power,awasthi2015efficient,zhang2018efficient,cortes2019region}. 

Besides the stream-based model where instances arrive one at a time, active learning has also been considered under the synthesized instances and the pool-based sampling models~\cite{settles2012active} and synthesizes instances for various models for applications (see for example,~\cite{Haupt2011, Hero2013, Lipor2017}. These models are less relevant to the online setting considered in this work.

\paragraph{Online Active Learning}

Active learning in the online setting has received much less attention. The work of \cite{cesa2003learning} and \cite{cavallanti2009linear} extended the margin-based approach to the online setting, focusing, as in the offline case, on homogeneous linear separators for instances on the unit sphere in $\mathbb R^d$. A specific noise model was adopted, which assumes that the underlying conditional distribution of the labels is fully determined by the Bayes optimal classifier $h^*$. In this work, we consider a general instance space and arbitrary classifiers. Tackling the general setting, the proposed algorithm and the analysis are fundamentally different from these two existing studies. Furthermore, we show in simulation examples that, even when restricted to the special case of homogeneous linear separators, the algorithm proposed in this work outperforms the margin-based algorithm developed in~\cite{cesa2003learning,cavallanti2009linear}.

The only work we are aware of that extends the disagreement-based approach to the online setting is~\cite{yang}, which extends the offline DHM algorithm to a stream-based setting.  In Sec.\ref{sec:1.2}, we discuss in detail the difference between~\cite{yang} and this work.   

In this work, we choose to adopt the disagreement-based design methodology. While approaches that more aggressively seek out informative labels may have an advantage in the offline setting when the learner can skip unlabeled instances at no cost and with no undesired consequences, such approaches may be less suitable in the online setting. The reason is that in the online setting, self labeling is required in the event of no query, classification errors need to be strictly constrained, and no feedback to the predicted labels is available (thus learning has to rely solely on queried labels).  These new challenges in the online setting are perhaps better addressed by the more conservative disagreement-based design principle that skips instances more cautiously. Simulation results in Sec.~\ref{sec:simulation} on the comparison with the margin-based algorithms corroborate this assessment.

\subsection{Main Results}
\label{sec:1.2}

We consider a general instance space $\mathcal X$, a general hypothesis space $\mathcal H$ of Vapnik-Chervonenkis (VC) dimension $d$ , and the Tsybakov noise model parameterized by $\alpha\in(0,1]$~\cite{tsybakov2004optimal}. We develop an online active learning algorithm and establish its $O(dT^{\frac{2-2\alpha}{2-\alpha}}\log^2 T)$  label complexity and uniformly bounded regret in prediction errors  with respect to the best classifier $h^*$ in $\mathcal H$. More specifically, the total expected classification errors in excess to $h^*$ over a horizon of length $T$ is bounded below $1/2$ independent of $T$, demonstrating that the proposed algorithm offers practically the same level of classification accuracy as $h^*$ with a sublinear label complexity in $T$. We further establish a matching (up to a poly-logarithmic factor) lower bound, demonstrating the order optimality of the proposed algorithm. We address the tradeoff between label complexity and regret and show that the algorithm can be modified to operate at a different point on the tradeoff curve. Below we contextualize this work with respect to the existing literature by highlighting the differences in three aspects: algorithm design, analysis techniques, and performance comparison.

\paragraph{Algorithm Design}
Referred to as OLA (OnLine Active), the algorithm developed in this work is rooted in the design principle of the disagreement-based approach.  The defining characteristic of the disagreement-based approach is to avoid querying instances that see insufficient disagreement among surviving hypotheses by maintaining, explicitly or inexplicitly, the RoD. Specific algorithm design differs in its temporal structure of when to update the RoD and, more crucially, in the threshold design on what constitutes sufficient disagreement. As detailed below, OLA differs from representative disagreement-based algorithms---the offline $A^2$~\cite{balcan2006agnostic} and DHM~\cite{dasgupta2008general} algorithms and the online ACAL algorithm~\cite{yang}---in both aspects.

In terms of temporal structure, OLA operates in epochs and updates the RoD at the end of each epoch, where an epoch ends when a fixed number $M$ of labels have been queried. This structure is different from $A^2$,  DHM, and ACAL. In particular, the epochs in $A^2$ are determined by the time instants when the size of the current RoD shrinks by half due to newly obtained labels. Such an epoch structure, however, requires the knowledge of the marginal distribution $\mathbb P_X$ of the instances for evaluating the size of the RoD. The epoch structure of OLA obviates the need for this prior knowledge. DHM, on the other hand, does not operate in epochs and updates (inexplicitly) the RoD at each time. Similarly, ACAL also updates the RoD at each time\footnote{ACAL has a predetermined epoch structure with geometrically growing epoch length. This epoch structure, however, is not for controlling when to update the RoD, but rather for setting a diminishing sequence of outage probability of eliminating $h^*$. The algorithm otherwise restarts by forgetting all past experiences at the beginning of each epoch.}. Moreover, the updates involve calculating thresholds by solving multiple non-convex optimization problems with randomized nonlinear constraints that can only be checked numerically. In contrast, the epoch-based updates in OLA only involve thresholds that are given in closed-form in terms of empirical errors.  

A more crucial improvement in OLA is the design of the threshold that determines the RoD. This is the key algorithm parameter that directly controls the tradeoff between label complexity and classification error rate.  By focusing only on empirical errors incurred over significant $(X,Y)$ examples determined by the current RoD, we obtain a tighter concentration inequality and a more aggressive threshold design, which leads to significant reduction in label complexity as compared with $A^2$, DHM, and ACAL, as well as margin-based algorithms (see details on the performance comparison below).


\paragraph{Analysis Techniques}
Under the offline PAC setting, the label complexity of an algorithm is often analyzed in terms of the suboptimality gap $\epsilon$ and the outage probability $\delta$. Under the online setting, however, the label complexity of an algorithm is measured in terms of the horizon length $T$, which counts both labeled and unlabeled instances. In the analysis of the label complexity of $A^2$~\cite{balcan2006agnostic,hanneke2007bound},  unlabeled instances are assumed to be cost free, and bounds on the number of unlabeled instances skipped by the algorithm are missing and likely intractable. Without a bound on the unlabeled data usage, the offline label complexity in terms of $(\epsilon,\delta)$ cannot be translated to its online counterpart. 

Yang~\cite{yang} analyze the label complexity by bounding the excess risk in terms of local Rademacher complexity~\cite{koltchinskii2006local} within each epoch. This technique is restricted to the specific threshold design in ACAL, which is based on expensive  non-convex optimization with constraints on randomized Rademacher process.

We adopt new techniques in analyzing the online label complexity of OLA. First we separate the analysis into two stages based on the size of the RoD.  For the early stage where the  RoD is large, we show that RoD is decreasing exponentially. Then, to upper bound the label complexity,  the key idea is to construct a supermartingale $\{S(t)\}_{t\ge 0}$ given by the difference of an exponential function of the total queried labels up to $t$ and a linear function of $t$. The optimal stopping theorem for supermartingales then leads to an upper bound on the exponential function of the label complexity. A bound on the label complexity thus follows from Jensen's inequality. The remaining label complexity where the RoD is small can be bounded by the product of the size and the remaining time horizon. The separation of the two stages is then optimized to tighten the bound. 

The lower bound established in this work is new. We are not aware of any existing lower bound on label complexity in the online setting. Lower bounds for the offline PAC setting (see, e.g.,~\cite{castro2008minimax,hanneke2015minimax}) are inapplicable to the online setting and were established using different techniques. 


\paragraph{Performance Comparison}
We now comment on the performance comparison in terms of both asymptotic orders and finite-time performance. 

As stated above, the performance analysis of $A^2$ is in terms of the PAC parameters $(\epsilon, \delta)$. The analysis of its online performance is missing. Dasgupta, et. al provided an upper bound on the unlabeled data usage in DHM~\cite{dasgupta2008general}. The bound, however, appears to be loose and translates to a linear $O(T)$ label complexity in the online setting. Yang~\cite{yang} provided an upper bound $O(dT^{\frac{2-2\alpha}{2-\alpha}}\log^3 T)$ on the label complexity of ACAL, which is higher than the $O(dT^{\frac{2-2\alpha}{2-\alpha}}\log^2 T)$ order offered by OLA.

The margin-based algorithm for learning homogeneous linear separators under a uniform distribution of $X$ on the unit sphere is analyzed in~\cite{cavallanti2009linear} under the Tsybakov noise condition. It leads to a regret order of $O(dT^{\frac{2-2\alpha}{3-2\alpha}}\log T)$ and a label complexity of $O(dT^{\frac{2-2\alpha}{2-\alpha}}\log T)$ under the Tsybakov low noise condition. These orders cannot be directly compared with that of OLA due to the restrictions to homogeneous linear separators and the specific form of $\mathbb P_{Y|X}$. This margin-based algorithm also operates at a different point on the tradeoff curve between regret and label complexity, offering a slightly lower order in label complexity but a higher order in regret.  However, even when restricted to the special case targeted by this margin-based algorithm, the dominating polynomial term is the same, and the finite-time comparison given by simulation examples in Sec.~\ref{sec:simulation} actually show superior performance of OLA in both label complexity and regret. 

The finite-time comparison in  Sec.~\ref{sec:simulation} also demonstrate significant performance gain offered by OLA over the three representative disagreement-based algorithms: $A^2$, DHM, and ACAL. In particular, the improvement over the online algorithm ACAL is drastic.

\section{Problem Formulation}
\label{sec:pre}

\subsection{Instances and Hypotheses}

Let $\{X_t\}_{t\ge 1}$ be a streaming sequence of instances, each drawn from an instance/sample space $\mathcal X$ and characterized by its feature vector.
Each subset of $\mathcal X$ is a concept. There is a target concept $\mathcal C\subset \mathcal X$ that the learner aims to learn (e.g., learning
the concept ``table'' from household objects).
Relating to the target concept $\mathcal C$, each instance $X_t$ has a hidden label $Y_t$, indicating whether $X_t\in \mathcal C$
(i.e., a positive example wherein $Y_t = 1$) or $X_t\notin \mathcal C$ (a negative example with $Y_t = 0$). The label $Y_t$
relates probabilistically to $X_t$ according to an unknown conditional distribution $\mathbb P_{Y |X}$.

The learner is characterized by its hypothesis space $\mathcal H$ consisting of all classifiers under consideration.
Each hypothesis $h\in \mathcal H$ is a measurable function mapping from $\mathcal X$ to $\{0, 1\}$. The
complexity of the hypothesis space $\mathcal H$ is measured by its VC dimension $d$.

\subsection{Error Rate, Disagreement, and Bayes Optimizer}

Recall that $\mathbb P_{Y|X}$ denotes the conditional distribution of the true label $Y$ for a given $X$. Let  $\mathbb P_X$ denote the unknown marginal distribution of instances $X$ and  $\mathbb P=\mathbb P_X \times \mathbb P_{Y|X}$ the joint distribution of an example~$(X,Y)$.
The error rate of a hypothesis $h$ is given by
\begin{equation}
\epsilon_\mathbb P (h) = \mathbb P [ h(X) \neq Y ],
\end{equation}
which is the probability that $h$ misclassifies a random instance.
Define the \emph{pseudo-distance} and the \emph{disagreement} between two hypotheses as, respectively,
\begin{equation}
\label{eq:rho}
d(h,h') = | \epsilon_{\mathbb P}(h) -   \epsilon_{\mathbb P}(h')| , ~~
\rho(h,h') = \mathbb P_X [h(X)\neq h'(X)],
\end{equation}
where the distance is the difference in error rates and the disagreement is the probability mass of the instances over which the two hypotheses disagree. Lastly, $\DIS(h_1, h_2) = \{x \in \mathcal{X}: h_1(x) \neq h_2(x)\}$ denotes the disagreement region between two hypotheses $h_1$ and $h_2$.

Let $h^*$ be the Bayes optimal classifier that minimizes the error rate, i.e., for all $x\in\mathcal X$, $h^*(x)$ is the label that minimizes the probability of classification error:
 \begin{equation}
 h^*(x)=\arg\min_{y=0,1}\mathbb E_{\mathbb P_{Y|X=x}}\mathbbm 1[Y\neq y],
 \end{equation}
where $\mathbbm 1[\cdot]$ is the indicator function. Let
\begin{equation}
 \label{eq:eta}
\eta(x)=\mathbb P_{Y|X=x}(Y=1|X=x).
\end{equation}
It is easy to see that
\begin{equation}
h^*(x)=\begin{cases}
1 \; \mbox{ if }  \eta(x)\ge \frac 12  \\
0 \; \mbox{ if }  \eta(x)<\frac 12
\end{cases}.
\end{equation}
We assume that $h^*\in \mathcal H$. 

\subsection{Noise Condition}

The function $\eta(x)$ given in \eqref{eq:eta} is a measure of the feature noise level at $x$.
The noise-free case is when labels are deterministic:  $\mathbb P_{Y|X=x}$, hence $\eta(x)$, assumes only values of $0$ and $1$. In this case, the optimal classifier $h^*$ is error-free. This is referred to as the realizable case with perfectly separable data.

In a general agnostic case with arbitrary $\mathbb P_{Y|X}$, consistent classifiers may not exist, and even $h^*$ suffers a positive error rate. A particular case, referred to as the Massart bounded noise condition~\cite{massart2006risk}, is when $\eta(x)$ is discontinuous at the boundary between positive examples $\mathcal X^*_1\triangleq \{x\in\mathcal X: h^*(x)=1\}$ and negative examples $\mathcal X^*_0\triangleq \{x\in\mathcal X: h^*(x)=0\}$. Specifically, there exists $\gamma>0$ such that $| \eta(x)-\frac 12| \ge \gamma$ for all $x\in\mathcal X$.

A more general noise model is the Tsybakov noise condition~\cite{tsybakov2004optimal}, for which the Massart bounded noise condition is a special case. It allows $\eta(x)$ to pass $\frac{1}{2}$ with a continuous change across the decision boundary and parameterizes the slope around the boundary. Specifically, the Tsybakov noise condition states that there exist $\alpha\in(0,1]$, $c_0\ge 0$, such that
for all $h$, we have
\begin{equation}
  \rho(h,h^*) \le c_0 d^\alpha(h,h^*). \label{eq:Tsybakov_noise_def}
\end{equation}
At $\alpha =1$, the Tsybakov noise reduces to the more benign Massart noise. In terms of the slope around the decision boundary, the above condition can be restated as 
\begin{equation}
  \mathbb{P}_{X}\left(\left\{ x: \left|\eta(x) - \frac{1}{2}\right| \leq \gamma \right\}\right) \leq c_0' \gamma^{\frac{\alpha}{1 - \alpha}}
\end{equation}
for some constant $c_0' \geq 0$.

\subsection{Learning Policies and Performance Measure}

An online active learning strategy $\pi$ consists of a sequence of query rules $\{\upsilon_t\}_{t\ge 1}$ and a sequence of prediction rules $\{\lambda_t\}_{t\ge 1}$, where $\upsilon_t$ and $\lambda_t$ map from causally available information consisting of past actions, instances, and queried labels to, respectively, the query decision of $0$ (no query) or $1$ (query) and a predicted label at time $t$. With a slight abuse of notation, we also let $\upsilon_t$ and $\lambda_t$ denote the resulting query decision and the predicted label at time $t$ under these respective rules.

The performance of policy $\pi=(\{\upsilon_t\}, \, \{\lambda_t\})$ over a horizon of length $T$ is measured by the expected number of queries and the expected number of classification errors in excess to that of the Bayes optimal classifier $h^*$. These two performance measures, referred to as label complexity~$\mathbb E[Q(T)]$ and regret~$\mathbb E[R(T)]$, are given as follows.

\begin{align}
\mathbb E[Q(T)] =& \mathbb E\left[\sum_{t=1}^T \mathbbm 1[\upsilon_t = 1]\right] \\
\mathbb E[R(T)] =&  \mathbb E\left[\sum_{t\le T:\upsilon_t=0} \mathbbm 1 [ \lambda_t \neq Y_t]  - \mathbbm 1 [ h^*(X_t) \neq Y_t]\right] ,
\end{align}

where the expectation is with respect to the stochastic process induced by $\pi$. Note that regret measures the expected difference in the \emph{cumulative} classification errors over the entire horizon between a learner employing $\pi$ and an oracle that uses $h^*$ all through the horizon.

The objective is a learning algorithm that minimizes the label complexity $\mathbb E[Q(T)]$ with a constraint on the regret $\mathbb E[R(T)]$. The constraint, for example, can be either bounded by a constant independent of~$T$ or in a logarithmic order of $T$. 

\section{The Online Active Learning Algorithm}
\label{sec:5}


\subsection{The Basic Structure}

The algorithm operates under an epoch structure. When a fixed number $M$ of labels have been queried in the current epoch, this epoch ends and the next one starts. Note that the epoch length, lower bounded by $M$, is random due to the real-time active query decisions.
The algorithm maintains two sets in each epoch $k$: the version space $\mathcal H_k$ and the RoD $\mathcal D(\mathcal H_k)$ defined as the region of instances for which there is disagreement among hypotheses in the current version space $\mathcal H_k$. More specifically,
\begin{equation}
\mathcal D(\mathcal H_k)\,=\,\{x\in\mathcal X: \exists h_1,h_2\in\mathcal H_k,~ h_1(x)\neq h_2(x)\}.
\label{eq:D}
\end{equation}
The initial version space is set to the entire hypothesis space $\mathcal H$, and the initial RoD is the instance space $\mathcal X$.
At the end of each epoch, these two sets are updated using the $M$ labels obtained in this epoch, and the algorithm iterates into the next epoch.

At each time instant $t$ of epoch $k$, the query and prediction decisions are as follows.
If $x_t\in \mathcal D(\mathcal H_{k})$,  its label is queried. Otherwise, the learner predicts the label of $x_t$ using an arbitrary hypothesis in $\mathcal H_{k}$.

At the end of the epoch, $\mathcal H_k$ is updated as follows.  Let $\mathcal Z_k$ denote the set of the $M$ queried examples in this epoch. For a hypothesis $h$ in~$\mathcal H_k$, define its empirical error over $\mathcal Z_k$ as
\begin{equation}
\label{eq:em_error}
\epsilon_{\mathcal Z_k}(h)=\frac{1}{M}\sum_{(x,y)\in \mathcal Z_k}\mathbbm 1[h(x)\neq y].
\end{equation}
Let $h^*_k= \arg\min_{h\in \mathcal H_k} \epsilon_{\mathcal Z_k}(h)$ be the best hypothesis in $\mathcal H_k$ in terms of empirical error
over $\mathcal Z_k$.
The version space is then updated by eliminating each hypothesis $h$ whose empirical error over $\mathcal Z_k$ exceeds that of $h^*_k$ by a threshold $\Delta_{\mathcal Z_k}(h,h_k^*)$ that is specific to $h$, $h_k^*$, and $\mathcal Z_k$. Specifically,
  \begin{equation}
 \label{eq:agHt}
 \mathcal H_{k+1}=\{h\in\mathcal H_k:\epsilon_{\mathcal Z_k}(h)-\epsilon_{\mathcal Z_k}( h_k^*)<\Delta_{\mathcal Z_k}(h,h_k^*)\}.
 \end{equation}
The new RoD $\mathcal D (\mathcal H_{k+1})$ is then determined by $\mathcal H_{k+1}$ as in~\eqref{eq:D}.

\subsection{Threshold Design}

 We now discuss the key issue of designing the threshold $\Delta_{\mathcal Z_k}(h,h_k^*)$ for eliminating suboptimal hypotheses. This elimination threshold controls the tradeoff between two conflicting objectives: quickly shrinking the RoD (thus reducing label complexity) and managing the irreversible risk of eliminating good classifiers (thus increasing future classification errors). 

 In OLA, we obtain a more aggressive threshold design
focusing on empirical errors incurred over significant $(X,Y)$ examples determined by the current RoD.

Specifically, for a pair of hypotheses $h_1,h_2$, define
\begin{equation}
\epsilon_{\mathbb P}(h_1,h_2)= \mathbb P( h_1(X)\neq Y\wedge h_2(X)=Y),
\end{equation}
which is the probability that $h_1$ misclassifies a random instance but $h_2$ successfully classified. For a finite set $\mathcal Z$ of $(x,y)$ samples, the empirical excess error of $h_1$ over $h_2$ on $\mathcal Z$ is defined as
\begin{equation}
\epsilon_{\mathcal Z}(h_1,h_2)\,\triangleq\,\frac{1}{|\mathcal Z|}\sum_{(x,y)\in \mathcal Z}\mathbbm 1[ h_1(x)\neq y\wedge h_2(x)=y].
\end{equation}
The elimination threshold $\Delta_{\mathcal Z_k}(h,h_k^*)$  is set to:
\begin{equation}
\begin{aligned}
\label{eq:thres}
\Delta_{\mathcal Z_k}(h,h_k^*) & = \beta_{\mathcal H_{k},M}^2 + \beta_{\mathcal H_k,M}\left(\sqrt{\epsilon_{\mathcal Z_k}(h, h_k^*)}+\sqrt{\epsilon_{\mathcal Z_k}( h_k^*,h)}\right),
\end{aligned}
\end{equation}
where $\beta_{\mathcal H',n}=\sqrt{(4/n)\ln (16T^2\mathcal S(\mathcal H',2n)^2)}$ for an arbitrary hypothesis space $\mathcal H'$ and positive integer $n$. Here $\mathcal S(\mathcal H',n)$ is the $n$-th shattering coefficient of $\mathcal H'$. By Sauer's lemma~\cite{bousquet2004introduction}, $\mathcal S(\mathcal H',n)=O(n^{d'})$ with $d'$ being the VC dimension of $\mathcal H'$.

The choice of this specific threshold function will become clear in Sec.~\ref{subsec:regret} when the relationship between the empirical error difference of two hypotheses and the ensemble error rate difference under $\mathbb P$ is analyzed. 

  \begin{algorithm}[h]
   \caption{The OLA Algorithm}
   \label{alg:oal}
\begin{algorithmic}
 \STATE {\bfseries  Input:} Time horizon $T$, VC dimension $d$, parameter $m\in\mathbb N^+$.
   \STATE {\bfseries  Initialization:} Set $\mathcal Z_1=\emptyset$, Version space $\mathcal H_1=\mathcal H$, RoD $\mathcal D_1=\mathcal X$. Current epoch $k=1$. $M=\lceil m dT^{\frac{2-2\alpha}{2-\alpha}}\log T \rceil$.

   \FOR{$t=1$ {\bfseries to} $T$}
   \IF{$x_{t}\notin \mathcal D_{k}$}
   \STATE Choose any $h\in \mathcal H_{k}$ and label $x_{t}$ with $h(x_{t})$;
   \ENDIF
   \IF{$x_{t}\in \mathcal D_{k}$}
   \STATE Query label $y_{t}$ and let $\mathcal Z_{k}=\mathcal Z_{k}\cup \{(x_t,y_t)\}$;
    \IF{$|\mathcal Z_k|=M$}
    \STATE Update $\mathcal H_{k+1}$ and $\mathcal D_{k+1}$ according to \eqref{eq:D} and \eqref{eq:agHt} with the elimination threshold $\Delta_{\mathcal Z_k}$ given in \eqref{eq:thres};
    \STATE Let $k=k+1$;

\ENDIF

   \ENDIF
   \ENDFOR

\end{algorithmic}
\end{algorithm}

A detailed description of the algorithm is given in Algorithm~\ref{alg:oal}. The algorithm parameter $M$ is set to $\lceil m dT^{\frac{2-2\alpha}{2-\alpha}}\log T \rceil$, where $m$ is a positive integer whose value will be discussed in Sec.~\ref{subsec:label}. We point out that while the horizon length $T$ is used as an input parameter to the algorithm, the standard doubling trick can be applied when $T$ is unknown.

\section{Analysis of Regret and Label Complexity}

We first develop the following concentration inequality in Theorem~\ref{lemma:3} to establish the relationship between the empirical error and ensemble error rate of any pair of hypotheses. The proof employs the normalized uniform convergence VC bound~\cite{vapnik2015uniform}. Details can be found in the appendix A.

\begin{theorem} 
\label{lemma:3}
Let $\mathcal Z$ be a set of $n$ i.i.d. $(X,Y)$-samples under distribution $\mathbb P$. For all $h_1,h_2\in \mathcal H$, we have, with probability at least $1-\delta$, 
\begin{equation}
\begin{aligned}
&\epsilon_{\mathbb P}(h_1)-\epsilon_{\mathbb P}(h_2) \le \epsilon_{\mathcal Z}(h_1)-\epsilon_{\mathcal Z}(h_2)\\
&+\gamma_n^2+\gamma_n(\sqrt{\epsilon_{\mathcal Z}(h_1,h_2)}+\sqrt{\epsilon_{\mathcal Z}(h_2,h_1)}),
\end{aligned}
\end{equation}
where $\gamma_n =\sqrt{(4/n)\ln (8\mathcal S(\mathcal H,2n)^2/\delta)}$.
\end{theorem}

Since all samples in $\mathcal D_k$ are queried at epoch $k$ in the proposed OLA algorithm, we can see that $\mathcal Z_k$ is an i.i.d. sample of size $M$ from distribution $\mathbb P|\mathcal D_k$, which is defined as
\begin{equation}
\mathbb P|\mathcal D_k(x)=\begin{cases}
{\mathbb P(x)}/{\phi(\mathcal D_k)} &\mbox{ if } x\in\mathcal D_k\\
0 & \mbox{ otherwise}
\end{cases},
\end{equation}
 where $\phi(\mathcal D)=\mathbb P(X\in \mathcal D)$ for $\mathcal D\subseteq \mathcal X$.
 
 Therefore, we can apply Theorem~\ref{lemma:3} to each epoch $k$ with $\mathcal Z_k$ and $\mathbb P|\mathcal D_k$, which gives us the following corollary.

\begin{corollary}
\label{lemma:4}
Let $\beta_n=\sqrt{(4/n)\ln (16T^2\mathcal S(\mathcal H,2n)^2)}$. With probability at least $1-\frac1 {2T}$, for all $k\ge 1$ and for all $h\in \mathcal H_k$, we have
\begin{equation}
\label{lemmathres}
\begin{aligned}
\epsilon_{\mathbb P|\mathcal D_k}( h_k^*)-&\epsilon_{\mathbb P|\mathcal D_k}(h)
\le\epsilon_{\mathcal Z_k}( h_k^*)-\epsilon_{\mathcal Z_k}(h)+\beta_{M}^2\\
+&\beta_{M}\left(\sqrt{\epsilon_{\mathcal Z_k}(h_k^*,h)}+\sqrt{\epsilon_{\mathcal Z_k}(h,h_k^*)}\right).
\end{aligned}
\end{equation}
\end{corollary}

\vspace{-0.1cm}
\subsection{Regret}
\label{subsec:regret}

Next, using Theorem~\ref{lemma:3} we show that the expected regret of the proposed OLA algorithm is bounded by $1/2$.

\begin{theorem}
\label{th:regret}
The expected regret $\mathbb E[R(T)]$ of the OLA algorithm is bounded as follows:
$$
\mathbb E[R(T)]\le\frac 12.
$$
 \end{theorem}

\begin{proof}
Here we provide the sketch of the proof. The detailed proof can be found in the Appendix B. First we show that if the inequalities in Corollary~\ref{lemma:4} hold simultaneously for all $k\ge 1$, we have $h^*\in \mathcal H_k$ for all $k\ge 1$, which implies $R(T)=0$. Therefore by Corollary 1, we have  $\mathbb P(R(T)>0)\le \frac {1}{2T}$. Note that $R(T)\le T$, we have $\mathbb E[R(T)]\le \frac {1}{2T} \cdot T = \frac12 $ as desired.
\vspace{-0.2cm}
\end{proof}

\subsection{Label Complexity}
\label{subsec:label}

For the purpose of label complexity analysis, we define the following online disagreement coefficient, which is slightly different from the disagreement coefficient defined for offline active learning in~\cite{hanneke2007bound}. 

Recall the psuedo-metric $\rho$ defined in~\eqref{eq:rho}.  The online disagreement coefficient $\theta=\theta(\mathbb P,\mathcal H)$ is defined as
\begin{equation}
\label{eq:theta}
\theta=\sup\left\{\frac{\phi[\mathcal D(B(h^*,r))]}{r}:r>0 \right\},
\end{equation}
where $B(h,r)=\{h\in\mathcal H:\rho(h,h')<r\}$  is a ``hypothesis ball'' centered at $h$ with radius $r$.

The quantity $\theta$ bounds the rate at which the disagreement mass of the ball $B(h^*,r)$ grows with the radius $r$. It is bounded by $\sqrt d$ when $\mathcal H$ is $d$-dimensional homogeneous separators~\cite{hanneke2007bound}.

Next we  upper bound the label complexity for the proposed online active learning algorithm. 

\begin{theorem} Let $\mathbb E[Q(T)]$ be the expected label complexity of OLA. If $m> 324 (\theta c_0)^\frac{2}{\alpha} $, then there exists $C_1>0$ such that    
\begin{equation}
\mathbb E[Q(T)]\le  C_1 md T^{\frac{2-2\alpha}{2-\alpha}}(\log T+1)^2,
\end{equation} 
where $\theta=\theta(\mathbb P_X,\mathcal H)$  is the disagreement coefficient.
\end{theorem}

Note that $m$ is a constant determined by the algorithm, the label complexity  $\mathbb E[Q(T)]$ has an order of $O(dT^{\frac{2-2\alpha}{2-\alpha}}\log^2 T)$. For the Massart noise condition at $\alpha=1$, the label complexity is $O(d\log^2 T)$.

\begin{proof}
We have discussed in Sec.~\ref{sec:1.2} the key ideas and techniques used in the proof. The detailed proof can be found in Appendix C. 


\end{proof}

\subsection{Order Optimality}
\label{subsec:lower_bound}

We now establish the order optimality of the label complexity of OLA under a bounded regret constraint. This is obtained by establishing a lower bound on the label complexity feasible under any policy with a bounded regret. 

\begin{theorem}
Consider the Tsybakov noise satisfying the following condition with a parameter $\alpha\in (0,1)$: there exist constants $c_1$ and $c_2$ independent of $x \in \mathcal{X}$ such that $\frac{c_1}{2} r_0(x)^{\frac{1}{\alpha} - 1} \leq \left| \eta(x) - \frac{1}{2} \right| \leq \frac{c_2}{2} r_0(x)^{\frac{1}{\alpha} - 1} $ holds for all $x \in \mathcal{X}$ where $\displaystyle r_0(x) = \inf_{\{h:  x \in \DIS(h, h^*)\} } \rho(h, h^*)$.
The label complexity of all policies with bounded regret is of order $\Omega(T^{\frac{2 - 2\alpha}{2 - \alpha}})$.
  \label{thm:lower_bound}
\end{theorem}

Note that a lower bound on the noise (i.e., an upper bound specified through the constant $c_2$ on the slope of $\eta(x)$ passing $1/2$) is further imposed in order to establish a tight lower bound on label complexity for a specific noise level. We point out that while we focused on the constraint of a bounded regret, the analysis can be easily modified to obtain lower bounds under regret constraints of different orders. Specifically, we can show a lower bound of $\Omega \left(\min \{ T^{2(1 - \alpha)(1 - \epsilon)}, T^{\frac{2 - 2\alpha}{2 - \alpha}} \} \right)$ under a regret constraint of order $O(T^{\epsilon})$ for some $\epsilon > 0$. It is also straightforward to modify the lower bound analysis to accommodate different problem models (e.g., those studied in~\cite{cavallanti2009linear,yang}).

\begin{proof} 
The key in establishing the lower bound is to identify a limiting subproblem inherent to the online classification problem that determines the label complexity. We show that an inherent binary hypothesis testing problem presents such a limit. For this specific subproblem, we show that the probability of the event where label complexity is capped at $\Omega(T^{\frac{2 - 2\alpha}{2 - \alpha}})$ is small if the regret on the subproblem has to be bounded. The detailed proof is given in Appendix D.
\end{proof}

For the case of Massart Noise, we can establish a lower bound of $\Omega(\log T)$ under the constraint of a sublinear regret budget. The basic proof technique follows similar ideas as that for the Tsybakov noise but with a simplified analysis (See Appendix D). We summarize the lower bound for the case of Massart Noise in the following theorem.

\begin{theorem}
	Consider the Massart Noise model where $\eta(x)$ is bounded away from $1/2$ by a parameter $\gamma > 0$, i.e., for all $x \in \mathcal{X}$, $| \eta(x) - 1/2 | \geq \gamma$. The label complexity of all policies that achieve a sublinear regret under the above Massart Noise model is of the order $\Omega(\log T)$.
	\label{thm:massart_lb}
\end{theorem}

For constraints of sublinear but unbounded regret, the above lower bound is tight since it matches with the upper bound on the label complexity of RW-OLA (see Sec.~\ref{sec:extensions_and_discussions}). Under the constraint of bounded regret, however, we conjecture a $\Omega(\log^2 T)$ lower bound on label complexity for a general hypothesis space with an infinite number of hypotheses. \footnote{The intuition behind this conjecture is as follows. Let $N(\epsilon)$ denote the $\epsilon$-covering number of $\mathcal{H}$ and $\mathcal{C}$ be an associated $\epsilon$-cover that has $N(\epsilon)$ hypotheses. Specifically, an $\epsilon$-cover $\mathcal{C}$ of $\mathcal{H}$ is a subset of hypothesis $\{h_1, \dots, h_N\}$ such that for any $h \in \mathcal{H}$ there exists an $i \in \{1, 2, \dots, N\}$ such that $\rho(h, h_i) \leq \epsilon$ and the $\epsilon$-covering number of $\mathcal{H}$ is the size of the smallest $\epsilon$-cover of $\mathcal{H}$. Following the same line of arguments in the proof of Theorem~\ref{thm:massart_lb}, we can show that for a hypothesis $h \in \mathcal{C}$, the policy needs to query $\Omega(\log T)$ instances in $\mathcal{D}(h, h^*)$ to ensure a bounded regret and this needs to hold simultaneously for all $h \in \mathcal{C}$ by the end of the time horizon. Using the uniformity of the cover $\mathcal{C}$, we can show that the expected number of queries to hit $\Omega(\log T)$ queries in $\DIS(h, h^*)$ for all $h \in \mathcal{C}$ is $\Omega(\log T \log N(\epsilon))$. Choosing $\epsilon \sim T^{-1/2}$ results in a bound of $\Omega(\log^2 T)$ on label complexity.}

\section{Extensions and Discussions}
\label{sec:extensions_and_discussions}
\subsection{Tradeoff Between Label Complexity and Regret}
In this section, we show that OLA can be modified to operate on a different point on the tradeoff curve of regret vs. label complexity.

In OLA, the threshold for elimination is constructed conservatively to achieve a bounded regret. More specifically, the outage probability of eliminating $h^*$ from the version space (i.e., the parameter $\delta$ in Theorem~\ref{lemma:3}) in each epoch is capped at a small value $1/T^2$ that diminishes with $T$. We now consider a variant of OLA in which the elimination probability $\delta$ is set to a constant in order to quickly shrink RoD for a lower label complexity. To mitigate the high probability of $h^*$ being eliminated, which may result in a linear regret order, we build in a verification stage at the beginning of each epoch for the algorithm to self recognize and recover from the event of $h^*$ being eliminated. The key idea is to devise a biased random walk on the version spaces that allows the algorithm to trace back to a previous version space in the event of $h^*$ being eliminated. We refer to this variant of OLA as RW-OLA.   

Before delving into the details of the verification stage, we define the parent and child relationship between version spaces. For each epoch $k$, if $ \mathcal H_k$ is obtained by eliminating some of the hypotheses in the version space $\mathcal H_{r(k)}$ of a previous epoch $r(k)$, we say that $\mathcal H_{r(k)}$ is the parent of $ \mathcal H_k$ and $ \mathcal H_k$ is a child of $\mathcal H_{r(k)}$. Note that $\mathcal H_k$ is a subset of $\mathcal H_{r(k)}$.

\paragraph{Verification Stage}
In the verification stage of epoch $k$, the query and prediction decision are based on its parent version space $\mathcal H_{r(k)}$ and its corresponding RoD. When a fixed number $M$ of labels have been queried, we start the verification process as follows.

Let $\mathcal Z_k'$ denote the set of the $M$ queried examples in the verification stage. We examine two values in terms of the empirical error over $\mathcal Z_k'$: (1) $\min_{h\in\mathcal H_k} \epsilon_{\mathcal Z_k'} (h)$: the minimum empirical error inside $\mathcal H_k$; (2) $\min_{h\notin\mathcal H_k} \epsilon_{\mathcal Z_k'} (h)$: the minimum empirical error outside $\mathcal H_k$. Intuitively, if $h^*\in\mathcal H_k$, the difference 
\begin{equation}
\label{empi}
\min_{h\notin\mathcal H_k} \epsilon_{\mathcal Z_k'} (h)- \min_{h\in\mathcal H_k} \epsilon_{\mathcal Z_k'} (h)
\end{equation}
will be large, and vice versa. We hence determine the outcome of the verification stage based on whether this gap between the empirical errors outside and inside the current version space $\mathcal{H}_k$ exceeds a properly designed threshold. If the verification passes, indicating that $h^*\in \mathcal{H}_k$ with a sufficiently high probability, the epoch proceeds in the same way as in OLA, and the current version space $\mathcal{H}_k$ is further pruned to form a new version space $\mathcal{H}_{k+1}$. If, on the other hand, the verification fails, the current epoch ends, and the algorithm traces back to the parent of $\mathcal{H}_k$ by setting $\mathcal H_{k+1} =\mathcal H_{r(k)}$. The evolution of the version spaces across epochs follows a biased random walk as detailed below. 

\paragraph{Random Walk on a Version-Space Tree}

Based on the outcome of the verification stage, the version space $\mathcal H_{k+1}$ of epoch $k+1$ is either a child or a parent of $\mathcal H_{k}$.
In particular, following the evolution of the version spaces, we can construct a growing tree to record the parent-children relationship between version spaces. In this tree structure, each node represents a version space, and the version space sequence $\{ \mathcal{H}_k \}_{k \geq 1}$ forms a random walk on the tree. Illustrated in Fig.~\ref{fig:random_walk_sample_path} is a sample path of the random walk on a tree for a hypothesis space consisting of threshold classifiers $\mathcal H=\{h_z|0\le z\le 1\}$ on $\mathcal X=[0,1]$ where $h_z=[z,1]$. We can see that on this tree, the verification failed in epochs $2$ and $5$ and but passed in epochs $1$, $3$, $4$, and $6$.

\begin{figure}[h]
\label{sp}
\centering
\includegraphics[scale=0.6]{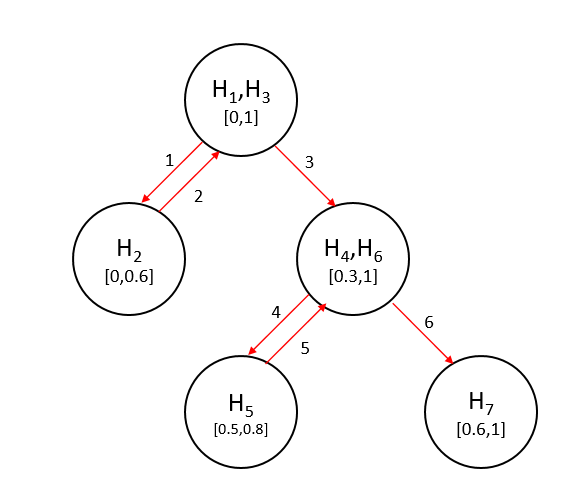} 
\caption{\small A typical random walk on a version space tree.}
\label{fig:random_walk_sample_path}
\end{figure}


\paragraph{Threshold design}

In addition to the elimination threshold for pruning the version space as in OLA, RW-OLA also requires a verification threshold. As explained below, these two thresholds are coupled and need to be designed jointly to ensure the desired performance of the algorithm. 

The verification stage performs a binary detection problem: whether $h^*$ is inside the current version space $\cH_k$. On one hand, to ensure that the random walk on the version spaces is biased toward the direction of correct pruning, the verification threshold needs to be chosen to ensure a sufficiently accurate detection outcome. On the other hand, the hardness of this detection problem is determined by how close the two cases of $h^*\in \cH_k$ and $h^*\notin \cH_k$ are. More specifically, the hardness of this binary detection problem is determined by the error rate difference between the best hypothesis inside $\mathcal H_k$ and the best hypothesis outside $\mathcal H_k$:

\begin{equation}
\label{true}
\min_{h\notin\mathcal H_k} \epsilon_{\mathbb P_{r(k)}} (h)- \min_{h\in\mathcal H_k} \epsilon_{\mathbb P_{r(k)}} (h),
\end{equation}
 which, in turn, is determined by the elimination threshold in epoch $r(k)$ when $\mathcal H_k $ is obtained. 
More specifically, while a smaller  elimination threshold leads to more aggressive pruning of the version space, it results in a smaller performance gap between hypotheses outside $\mathcal{H}_k$ and those inside $\mathcal{H}_k$ since near-optimal hypotheses may be eliminated from the version space. Consequently, the verification stage faces a harder detection problem. In summary, the two thresholds need to be designed jointly to balance the label complexity associated with verification and with normal learning.  
%
%

Let $p$ denote the desired bias of the random walk. This implies that (i) when $h^*\in \mathcal H_k$, the verification passes with a probability no smaller than $p$; (ii) when $h^*\notin \mathcal H_k$, the verification fails with a probability no smaller than $p$.  Let
\begin{equation}
\Delta(M,\delta) = 2\sqrt{2\frac{\log \mathcal S(\mathcal H,2M)+\log \frac 2\delta}{M}}.
\end{equation}

We set the elimination threshold to $6\Delta(M,1-\sqrt p)$ so that the error rate difference in (\ref{true}), which determines the hardness of the verification problem, is at least $4\Delta(M,1-\sqrt p)$ with high probability. The verification threshold is set to $2\Delta(M,1-\sqrt p)$ to guarantee the desired bias of $p$. A detailed derivation of the thresholds is given in Appendix E.

\begin{algorithm}
   \caption{The Random Walk OLA (RW-OLA) Algorithm}
   \label{alg:rwoal}
\begin{algorithmic}
 \STATE {\bfseries  Input:}  VC dimension $d$, parameter $m\in\mathbb N^+$.
   \STATE {\bfseries  Initialization:} Set Version space $\mathcal H_1=\mathcal H$, RoD $\mathcal D_1=\mathcal X$. Current epoch $k=1$. $M= m d $. Parents $r(1)=1$

   \WHILE{$t\le T$}
	\STATE {\bfseries Verification:}
	\STATE Let $Z_{k}' = \emptyset$
	\WHILE{$| Z_{k}'| < M$}
		\STATE Let $t=t+1$
	\IF {$x_t \notin \mathcal D_{r(k)}$}
		\STATE Choose any $h\in \mathcal H_{r(k)}$ and label $x_{t}$ with $h(x_{t})$;
	\ELSE
		\STATE Query label $y_{t}$ and let $\mathcal Z_{k}'=\mathcal Z_{k}'\cup \{(x_t,y_t)\}$;
	\ENDIF
	\ENDWHILE
			\IF{$\min_{h\notin\mathcal H_k} \epsilon_{\mathcal Z_k'} (h) - \min_{h\in\mathcal H_k} \epsilon_{\mathcal Z_k'} (h) < 2\Delta (M, 1-p)$}
			\STATE Let $\mathcal H_{k+1}=\mathcal H_{r(k)}$, $\mathcal D_{k+1}=\mathcal D_{r(k)}$, $r(k+1) = r(r(k))$, $k\leftarrow k+1$;
			\STATE continue;
			
			\ENDIF
	
   \STATE {\bfseries Elimination:}
	\STATE Let $Z_{k} = \emptyset$
\WHILE{$| Z_{k}| < M$}
		\STATE Let $t=t+1$
	\IF {$x_t \notin \mathcal D_{k}$}
		\STATE Choose any $h\in \mathcal H_{k}$ and label $x_{t}$ with $h(x_{t})$;
	\ELSE
		\STATE Query label $y_{t}$ and let $\mathcal Z_{k}=\mathcal Z_{k}\cup \{(x_t,y_t)\}$;
	\ENDIF
	\ENDWHILE

   	\STATE Update $\mathcal H_{k+1}$ as following:
\begin{equation}
 \mathcal H_{k+1}=\{h\in\mathcal H_k:\epsilon_{\mathcal Z_k}(h)-\epsilon_{\mathcal Z_k}( h_k^*)<6\Delta(M,1-\sqrt p)\}
 \end{equation}
\STATE Update $\mathcal D_{k+1}$ according to (10).
      \STATE Let $r(k+1)=k$, $k=k+1$;

   \ENDWHILE

\end{algorithmic}
\end{algorithm}



A detailed description of the algorithm is given in Algorithm~\ref{alg:rwoal}. The algorithm parameter $M$ is set to $\lceil m d \rceil$, where $m$ is a positive integer whose value will be discussed in the analysis below. \\

\paragraph{Analysis of Regret and Label Complexity}
\label{rw_ana}


\begin{theorem} Let $\mathbb E[Q(T)]$ be the expected label complexity of the RW-OLA algorithm. If $m> 1024 (\theta c_0)^2$, under Massart noise condition, there exists $C_2>0$ such that    
\begin{equation}
\mathbb E[Q(T)]\le  C_2 md \log T,
\end{equation} 
where $\theta=\theta(\mathbb P_X,\mathcal H)$  is the disagreement coefficient.
\label{thm:lc_RW_OLA}
\end{theorem}

\begin{proof}
Here we provide a sketch of the proof. The detailed proof can be found in the Appendix E. We first show that the bias of the random walk is indeed bounded above $p$ with the chosen thresholds. We then show that when the verification passes, the RoD in the next epoch will shrink with a fixed rate $c$. Based on these two statements, we can show that the expected RoD is decreasing exponentially with rate $c_1=1-p+c^2p<1$. The same submartingale technique used in analyzing OLA is then used to bound the label complexity of RW-OLA.

\end{proof}
Under Massart noise, the epoch length for OLA is $md\log T$, which leads to $O(\log^2 T)$ label complexity. For RW-OLA, the epoch length is only $md$, which makes the label complexity only $O(\log T)$. In other words, to make sure RoD decreases exponentially, OLA requires epoch length to be $O(d\log T)$ but RW-OLA only requires it to be $O(1)$.


\begin{theorem} Let $\mathbb E[R(T)]$ be the expected regret of the RW-OLA algorithm. If $m> 1024 (\theta' c_0)^2$ and $\theta'>0$, under Massart noise condition, there exists $C_3>0$ such that    
\begin{equation}
\mathbb E[R(T)]\le  C_3 md \log T,
\end{equation} 
where 
\begin{equation}
\label{eq:theta_prime}
\theta'=\inf\left\{\frac{\phi[\mathcal D(B(h^*,r))]}{r}:r>0 \right\}.
\end{equation}  
is the modified disagreement coefficient.
\label{thm:reg_RW_OLA}
\end{theorem}

\begin{proof}
Since an epoch $k$ with $h^*\in\mathcal H_k$ incurs no regret, we only need to consider the case where $h^*\notin \mathcal H_k$. In this case, based on the RW-OLA algorithm, regret incurs if and only if the instance falls into a subset outside of RoD. Based on the bias of the random walk, we can show that the expected ratio of that subset to the current RoD is bounded by a constant. Since queries occur whenever the instances fall inside the RoD, we can show that the expected regret is upper bounded by this constant multiplying the expected label complexity, which is $O(d\log T)$. See Appendix F for the detailed proof.
\end{proof}

\subsection{Implementation for Homogeneous Linear Classification}
\label{sec:3.3}
 There are several steps in OLA and RW-OLA that can be computational expensive, which is inherent to the disagreement-based approach. Specifically, maintaining the version space and RoD, and computing the best empirical hypothesis $h^*_k$ can be costly. We discuss here approximate implementations with manageable computational complexity for homogeneous linear classification, drawing inspiration from techniques of using surrogate loss \cite{hanneke2012surrogate} and the Query-by-Committee approaches \cite{freund1997selective,dasgupta2005analysis}.

In homogeneous linear classification, $\mathcal X $ is the surface of the $d$-dimension unit Euclidean sphere. Each hypothesis, as a linear separator that passes the origin, is given by a unit vector $\mathbf u \in \mathbb R^d$ such that the corresponding concept is $\{\mathbf x\in\mathcal X: \mathbf u\mathbf x \ge 0\}$.

To estimate the best empirical hypothesis $h^*_k$, we use hinge loss function $l(z) = \max \{1-z,0\}$ to replace the 0-1 loss function in~(\ref{eq:em_error}). Then,  the best empirical hypothesis $\hat {h}^*_k$ under hinge loss is given by
\begin{equation}
\label{eq:hinge_em_error}
\hat {h}^*_k=\min_{h\in\mathcal H_k}\sum_{(x,y)\in \mathcal Z_k} \max \{1-(2y-1)\mathbf u x,0\}.
\end{equation}
Then standard  linear classification algorithms such as SVM can be employed to compute  $\hat {h}^*_k$.

The version space is approximated with $N$ constituent hypotheses sampled uniformly at random. Specifically, at  $t=1$, we  sample $N$ hypotheses uniformly at random from the entire hypothesis space $\{\mathbf u\in \mathbb R^d, ||\mathbf u||=1\}$ and form $\hat {\mathcal H_1}$. At each epoch $k$, for each hypothesis $h\in\hat {\mathcal H}_{k}$, we check whether it should be eliminated based on (\ref{eq:agHt}) and label them as +1 or -1 accordingly. Then, we run a linear classification algorithm to find a linear classifier $\mathbf \omega \mathbf u+b\ge 0$ that separates them. To obtain an approximate of 
 the new version space $\hat {\mathcal H}_{k+1}$,  we again sample $N$ hypotheses uniformly at random from $\{\mathbf u:\mathbf \omega \mathbf u+b\ge 0\}$ to form the next version space $\hat {\mathcal H}_{k+1}$.

Since the version space is estimated by a finite number of hypotheses, instead of maintaining the RoD explicitly, we check whether $x_t\notin \mathcal D_k$  by checking whether all $h\in\hat {\mathcal H}_{k}$ agree on $x_t$. A detailed description of the algorithm is given in Algorithm~\ref{alg:oal2}. 

 For RW-OLA, both the verification stage and elimination stage can be implemented similarly. In particular, the elimination stage of RW-OLA is exactly the same as OLA except the threshold will be different. Therefore, it can be done by replacing the threshold in step 2 to maintain the version space and RoD. For the verification stage, which involves finding the best empirical hypothesis inside and outside of $\mathcal H_k$, can be done using the hinge loss replacement in (\ref{eq:hinge_em_error}) with a standard linear classification algorithm as well.


  \begin{algorithm}[h]
   \caption{ OLA for Homogeneous Linear Classification}
   \label{alg:oal2}
\begin{algorithmic}
   \STATE {\bfseries  Initialization:} Set $\mathcal Z_0=\emptyset$, Random sample  $\hat {\mathcal H}_{0}$ uniformly from $\mathcal H$. 

   \FOR{$t=1$ {\bfseries to} $T$}
   \IF{All $h\in\hat {\mathcal H}_{k}$ agree on $x_t$}
   	\STATE Choose any $h\in \hat {\mathcal H}_{k}$ and label $x_{t}$ with $h(x_{t})$;
   \ELSE
   	\STATE Query label $y_{t}$ and let $\mathcal Z_{k}=\mathcal Z_{k}\cup \{(x_t,y_t)\}$;
    	\IF{$|\mathcal Z_k|=M$}
		\STATE 1. Find $h^*_k$ using $\mathcal Z_k$ and (\ref{eq:hinge_em_error});
		\STATE 2. For all $h\in\hat {\mathcal H}_{k}$, check whether $h\in {\mathcal H}_{k+1}$ based on (\ref{eq:agHt}) and label them accordingly;
		\STATE 3. Find linear classifier $\mathbf \omega \mathbf u+b\ge 0$ for $\hat {\mathcal H}_{k}$ and its label;
		\STATE 4. Random sample $\hat {\mathcal H}_{k+1}$ from $\{\mathbf u:\mathbf \omega \mathbf u+b\ge 0\}$;

	\ENDIF

   \ENDIF
   \ENDFOR

\end{algorithmic}
\end{algorithm}

\section{Simulation Examples}
\label{sec:simulation}

We first compare the label complexity of OLA and RW-OLA with existing disagreement-based active learning algorithms. We first consider a one-dimensional instance space $\mathcal X=[0,1]$ and threshold classifiers with $\mathcal H=\{h_z|0\le z\le 1\}$ where $h_z=[z,1]$. Note that the VC dimension $d=1$.  We set $\mathbb P_X$ to be the uniform distribution. Figure 2 and 3 show the comparison under different Tsybakov noise conditions. 

\begin{figure}[h]
\label{fig:2}
\centering
\includegraphics[scale=0.5]{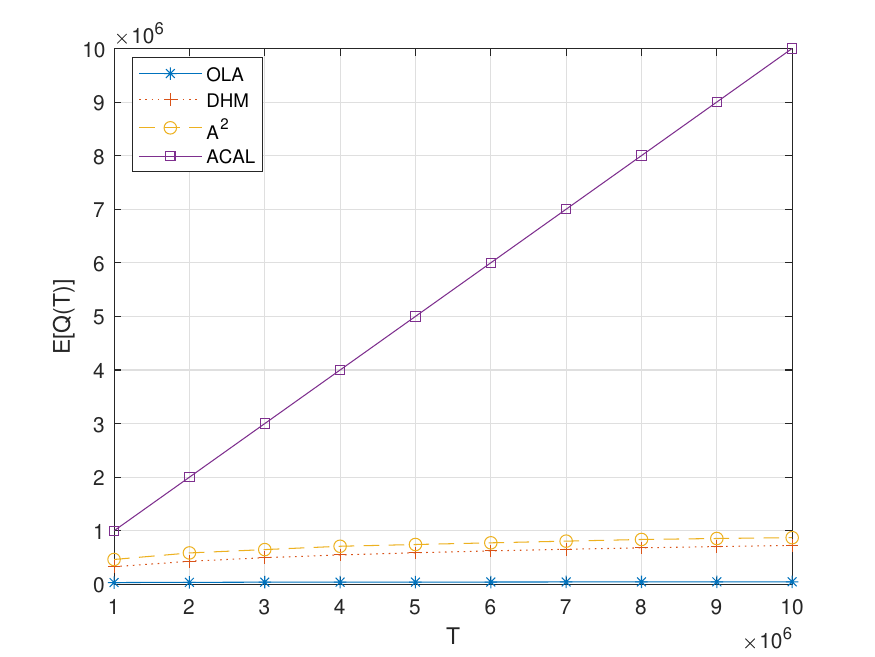} 
\includegraphics[scale=0.5]{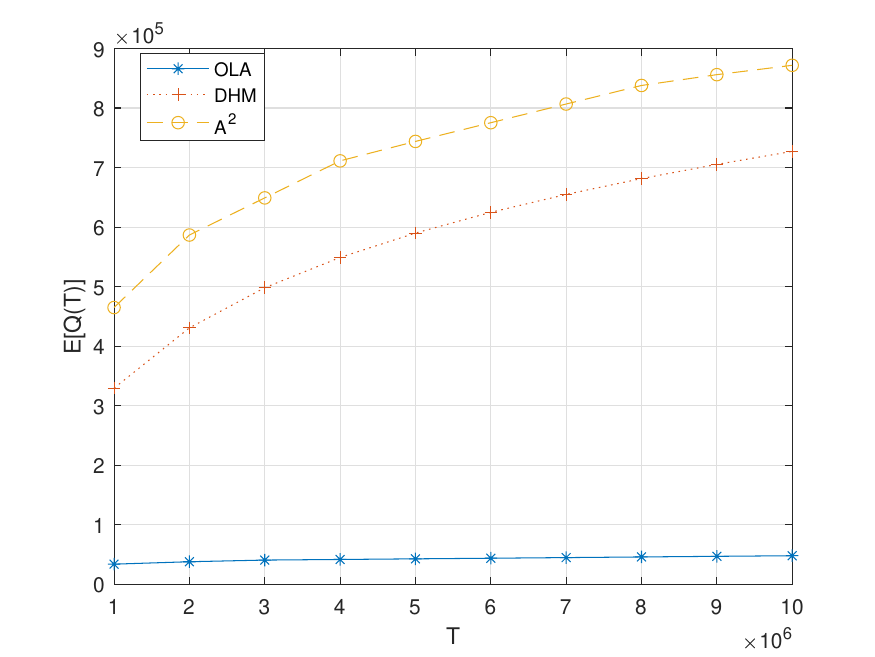}  
\caption{\small Comparison with $A^2$, DHM, and ACAL ($d=1$, Tsybakov noise with $\alpha=1$ and $c_0=5$, $h^*=h_{0.5}$).}
\end{figure}

\begin{figure}[h]
\label{fig:3}
\centering
\includegraphics[scale=0.5]{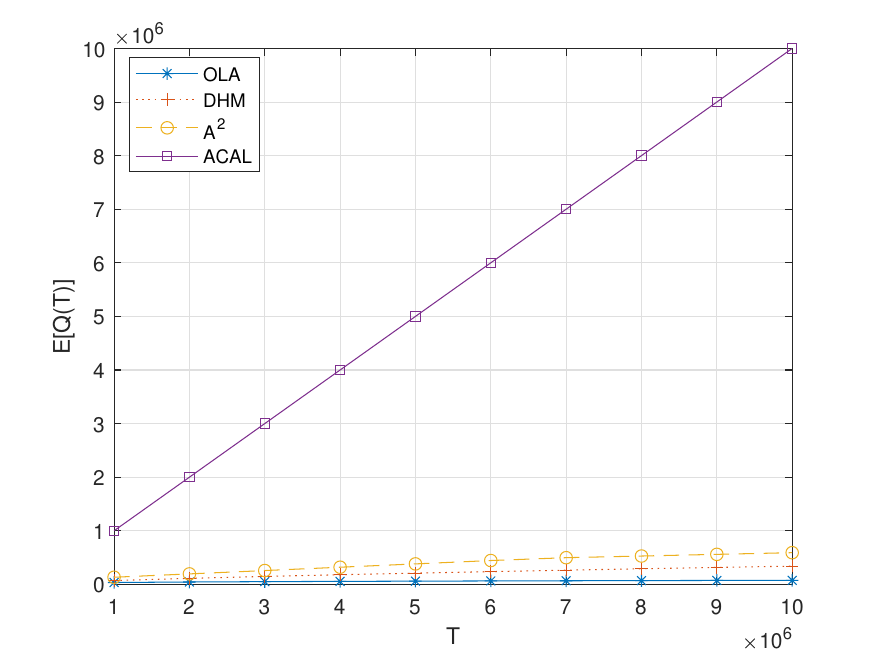} 
\includegraphics[scale=0.5]{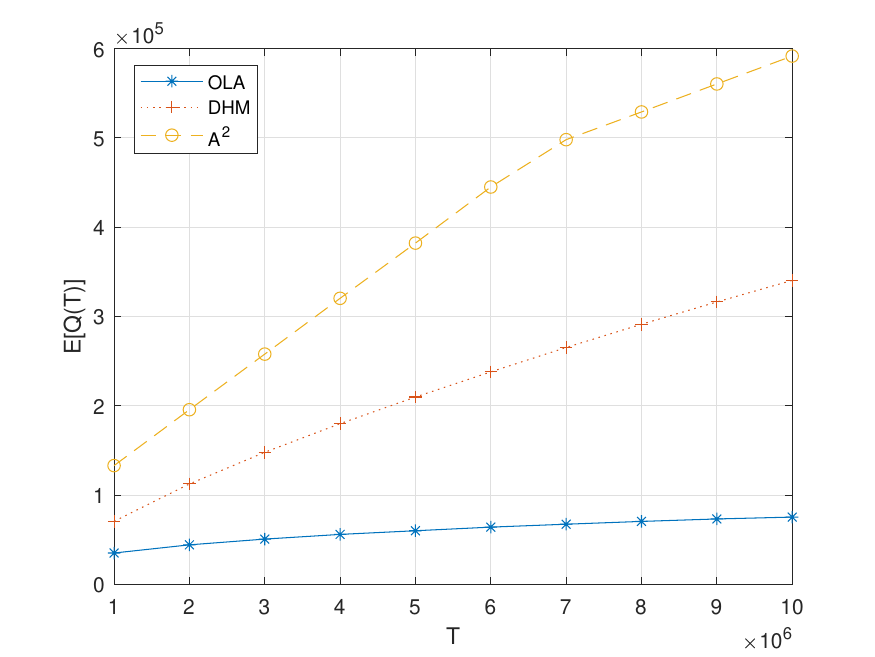}  
\caption{\small Comparison with $A^2$, DHM, and ACAL ($d=1$,  Tsybakov noise with $\alpha=0.5$ and $c_0=1$, $h^*=h_{0.5}$).}
\end{figure}

\begin{figure}[h]
\centering
\includegraphics[scale=0.5]{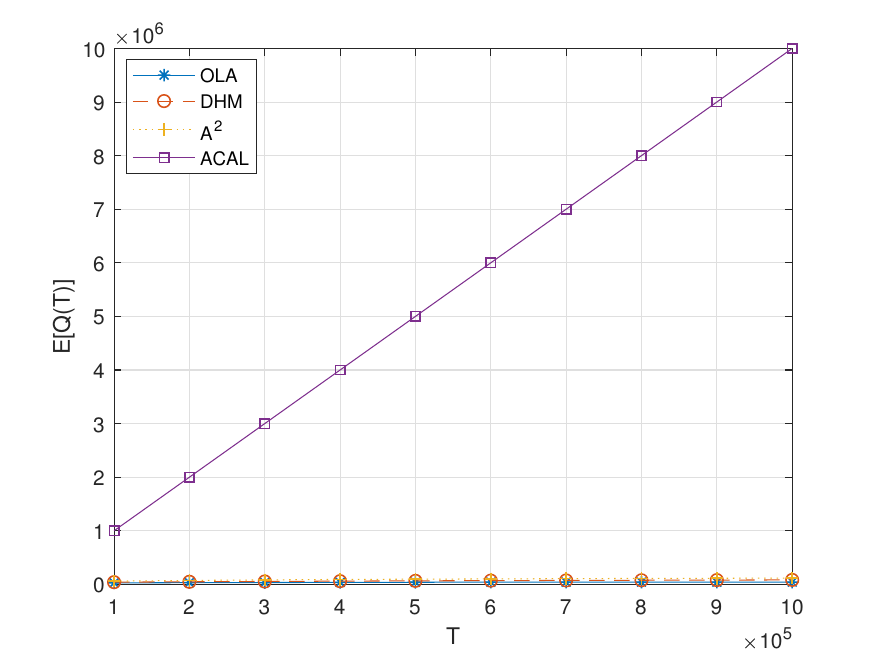} 
\includegraphics[scale=0.5]{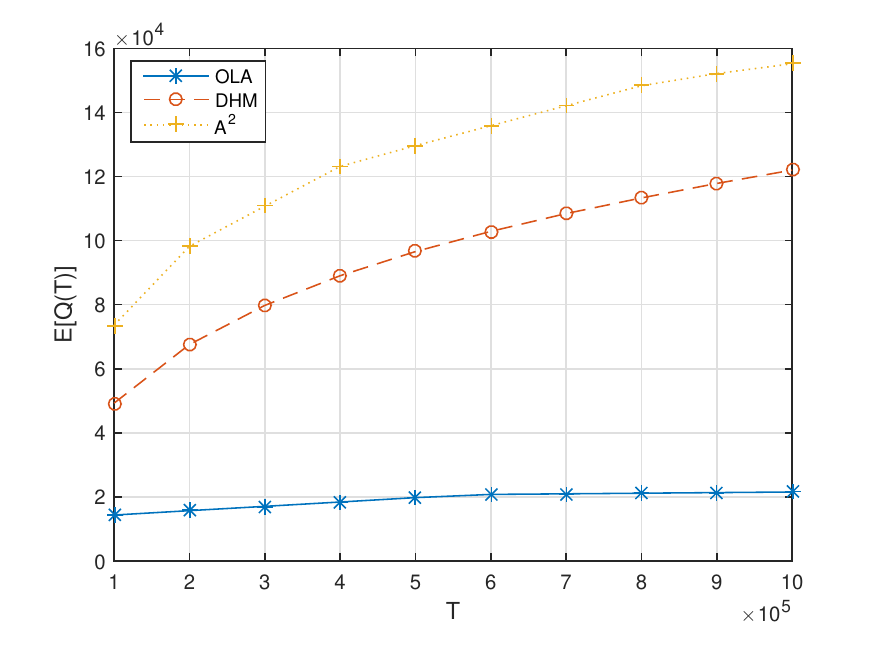}  
\caption{\small Comparison with $A^2$, DHM, and ACAL for ($d=2$, Tsybakov $\alpha=1$ and $c_0=1$ $h^*=h_{0,25,0.75}$).}
\end{figure}

In Figure 4, we consider the same instance space $\mathcal X=[0,1]$ and uniformly distributed instances, but a hypothesis space $\mathcal H=\{h_{z_1,z_2}|0\le z_1,z_2\le 1\}$ consisting of all intervals $h_{z_1,z_2}=[z_1,z_2]$. Note that in this case, the VC dimension $d=2$.

Since the label complexity for ACAL is much larger than the others, we plot the others in the right figure. The significant reduction in label complexity offered by OLA and RW-OLA is evident from Figures 2-4. The simulated classification errors are near zero for all the algorithms.

Next we consider $d$-dimension homogeneous linear classification setting. Figure 5 and 6 show the comparison  for $d=3, 4$. DHM and $A^2$ are implemented with similar methods as discussed in Sec. \ref{sec:3.3}. The simulated classification errors are near zero for all three algorithms.

\begin{figure}[h]
\label{fig:5}
\centering
\includegraphics[scale=0.5]{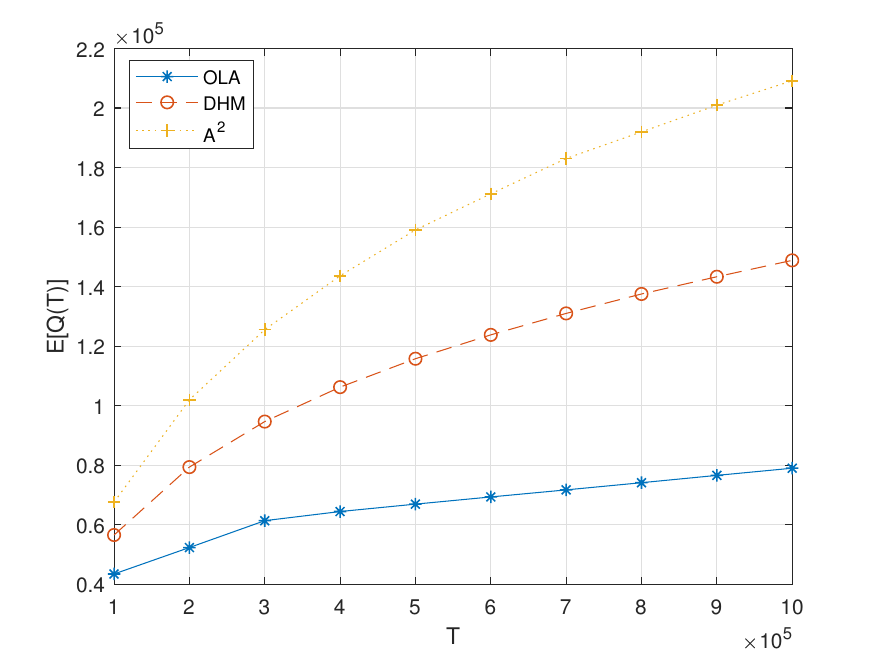} 
\includegraphics[scale=0.5]{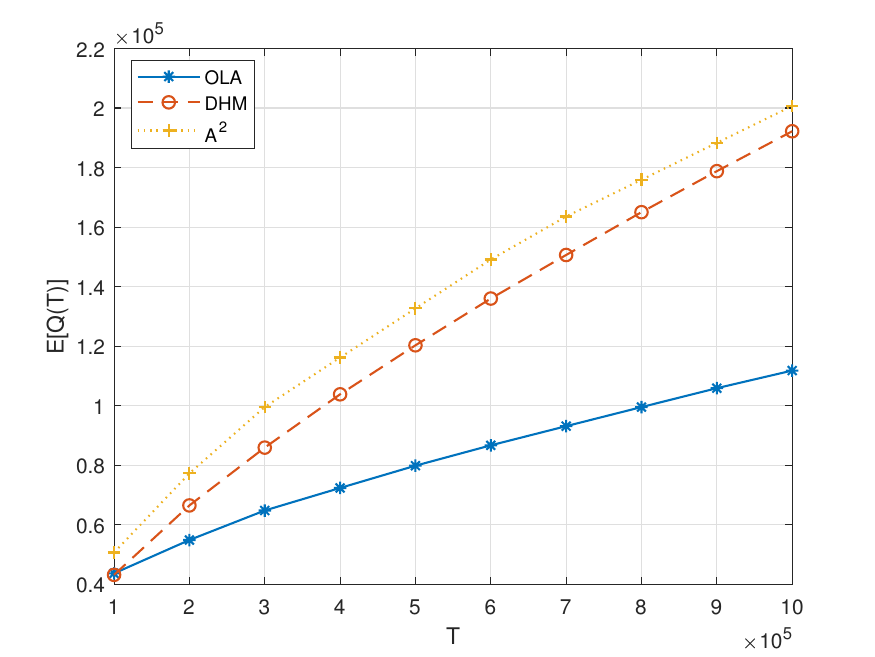}  
\caption{\small Comparison with $A^2$ and DHM for ($d=3$, $N=50000$, Tsybakov noise with $\alpha=1$ and $c_0=1$,  $\alpha=0.5$ and $c_0=5$, $h^*=(1,0,0)$).}
\end{figure}

\begin{figure}[!h]
\label{fig:4}
\centering
\includegraphics[scale=0.5]{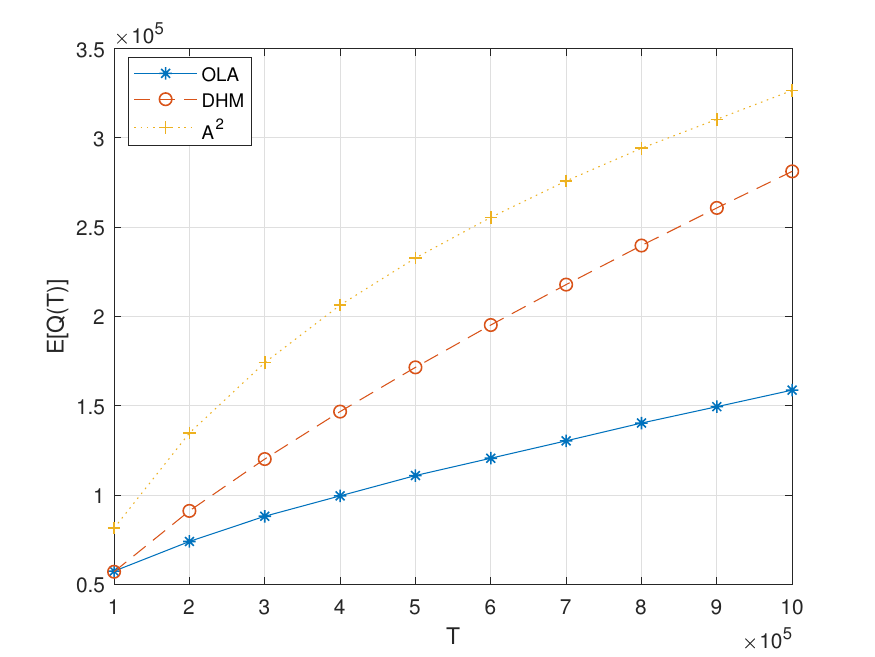} 
\includegraphics[scale=0.5]{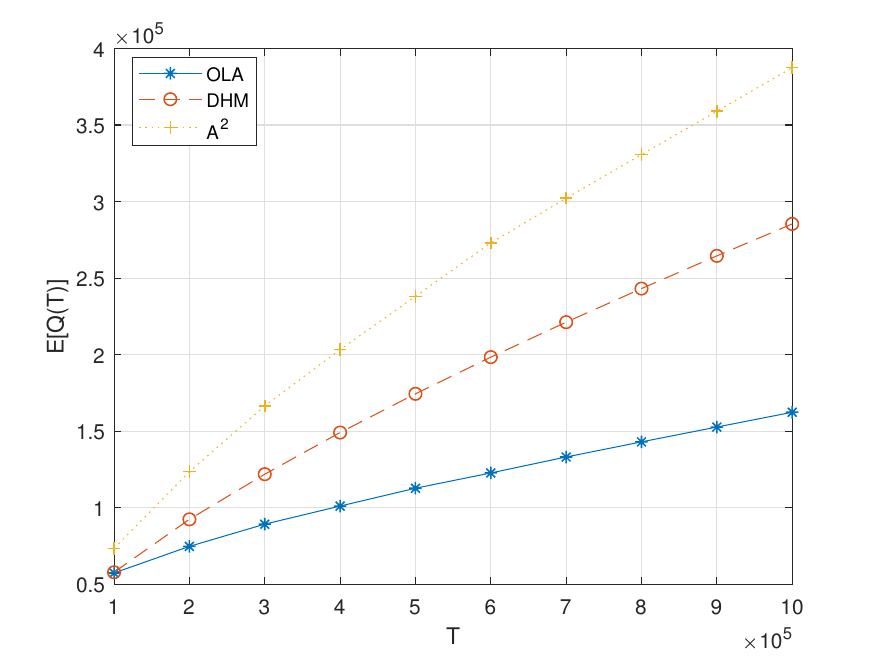}  
\caption{\small Comparison with $A^2$ and DHM for ($d=4$, $N=50000$, Tsybakov noise with  $\alpha=1$ and $c_0=1$,  $\alpha=0.5$ and $c_0=5$ $h^*=(1,0,0,0)$).}
\end{figure}

Next we compare OLA with the online margin-based algorithm CB-C-G proposed in \cite{cesa2003learning,cavallanti2009linear}. It is specialized in learning homogeneous separators under specific noise model: there exists a fixed and unknown vector $\mathbf u\in\mathbb R^d$ with Euclidean norm $||\mathbf u||$ = 1 such that $\eta(\mathbf x)=(1+\mathbf u^\top \mathbf x)/2$. Then, the Bayes optimal classifier $h^*(\mathbf x) = \mathbbm 1[\mathbf u^\top \mathbf x\ge 0]$. Shown in Figure 7 are the label complexity and classification error comparisons under this specific noise model with $d=2, \mathbf u=(1,0)$, and uniform $\mathbb P_X$.  It shows that even when comparing under this special setting, OLA offers considerable reduction in label complexity and drastic improvement in classification accuracy. This confirms with the assessment discussed in Sec.~\ref{sec:intro} that the more conservative disagreement-based approach is more suitable in the online setting than the more aggressive margin-based approach.  

\begin{figure}[!h]
\label{cbcg}
\centering
\includegraphics[scale=0.5]{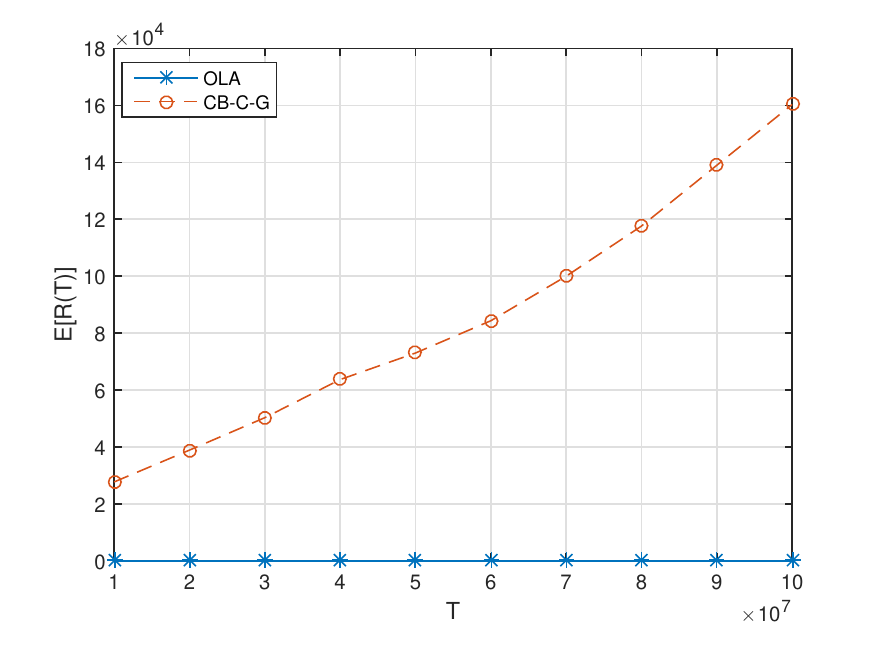}  
\includegraphics[scale=0.5]{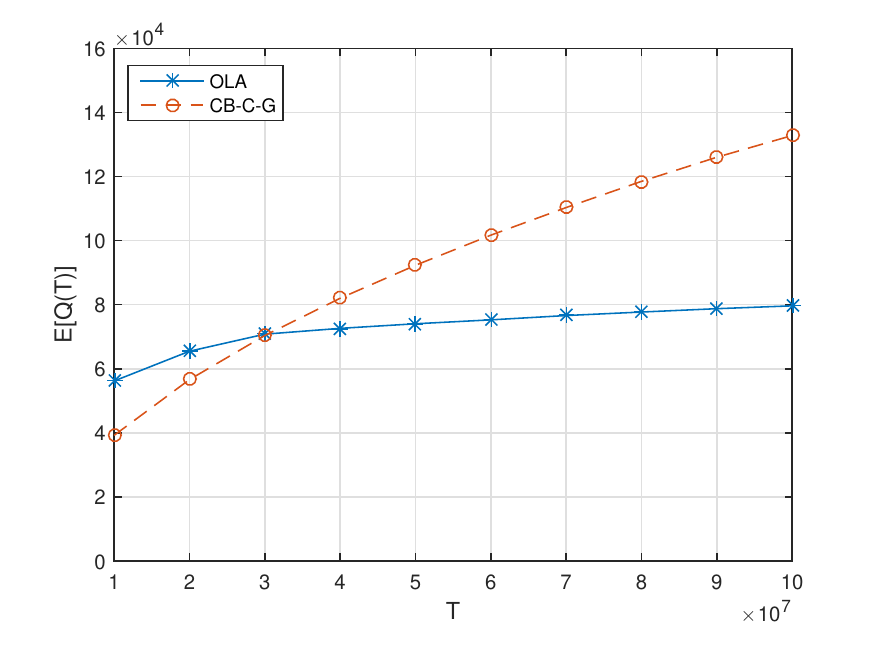}  
\caption{\small Comparison with CB-C-G~\cite{cesa2003learning} under Tsybakov noise with $\alpha=0.5$ and $c_0=1$.}
\end{figure}

\section{Conclusion}
\label{sec:conclusion}
Online active learning has received considerably  less attention than its offline counterpart. Real-time stream-based applications, however, necessitate a better understanding of this problem. The proposed algorithms and the established lower bounds in this work represent only initial attempts at addressing this problem. Much remains open. In particular, the characterization of the regret vs. label complexity tradeoff is incomplete,  and online learning algorithms that can operate at any given point on the tradeoff curve require further investigation. 



\newpage

\newpage

\section*{Appendix A: proof of Theorem~1}
First we introduce the following normalized uniform convergence VC bound [Vapnik and Chervonenkis, 2015].

\begin{lemma}
\label{lemma:VCbound}
Let $\mathcal F$ be a family of measurable functions $f:\mathcal X\times \mathcal Y\to \{0,1\}$. Let $\mathbb Q$ be a fixed distribution over $\mathcal X\times \mathcal Y$. Define
\begin{equation}
\mathbb Qf = \mathbb E_{X,Y\sim\mathbb Q} f(X,Y).
\end{equation}
For a finite set $\mathcal Z\subseteq\mathcal X\times \mathcal Y$, define
\begin{equation}
\mathbb Q_{\mathcal Z}f = \frac{1}{|\mathcal Z|}\sum_{(X,Y)\in\mathcal Z} f(X,Y)
\end{equation}
as the empirical average of $f$ over $\mathcal Z$. If $\mathcal Z$ is an i.i.d. sample of size $n$ from $\mathbb Q$, then, with probability at least $1-\delta$, for all $f\in\mathcal F$:
\begin{equation}
\label{eq:probVC}
\gamma_n\sqrt{\mathbb Q_\mathcal Z f}\le \mathbb Q f-\mathbb Q_\mathcal Z f\le\ \gamma_n^2+\gamma_n\sqrt{\mathbb Q_\mathcal Z f},
\end{equation}
where $\gamma_n=\sqrt{(4/n)\ln(8\mathcal S(\mathcal F,2n)/\delta)}$.
\end{lemma}
Define
\begin{equation}
\label{eq:fh1h2}
g^{h_1}_{h_2}(x,y)=\mathbbm 1[ h_1(x)\neq y\wedge h_2(x)=y],
\end{equation}
which is a measurable mapping from $\mathcal X \times \mathcal Y$ to $\{0,1\}. $ It is not hard to see that $\mathbb P g^{h_1}_{h_2} = \epsilon_{\mathbb P}(h_1,h_2)$ and $\mathbb P_\mathcal Z g^{h_1}_{h_2} = \epsilon_{\mathcal Z}(h_1,h_2)$.

Then, we apply the normalized VC bound [Vapnik and Chervonenkis, 2015] to a family of measurable functions $\mathcal F=\{g^{h_1}_{h_2}|h_1,h_2\in \mathcal H\}=\{g^{h_2}_{h_1}|h_1,h_2\in \mathcal H\}$ defined in~\eqref{eq:fh1h2} which gives us the following two inequalities:

\begin{equation}
\label{eq:20}
\epsilon_{\mathbb P}(h_1,h_2)-\epsilon_{\mathcal Z}(h_1,h_2) \le \gamma_n^2+\gamma_n\sqrt{\epsilon_{\mathcal Z}(h_1,h_2)}
\end{equation}
and
\begin{equation}
\label{eq:21}
\epsilon_{\mathcal Z}(h_2,h_1)-\epsilon_{\mathbb P}(h_2,h_1) \le \gamma_n\sqrt{\epsilon_{\mathcal Z}(h_2,h_1)}.
\end{equation}
Since 
\begin{equation}
\begin{aligned}
\mathbbm 1[ h_1(x)\neq y] - \mathbbm 1[ h_2(x)\neq y] = \mathbbm 1[ h_1(x)\neq y\wedge h_2(x)=y]-\mathbbm 1[ h_2(x)\neq y\wedge h_1(x)=y],
\end{aligned}
\end{equation}
we have $\epsilon_{\mathcal Z}(h_1)-\epsilon_{\mathcal Z}(h_2)=\epsilon_{\mathcal Z}(h_1,h_2)-\epsilon_{\mathcal Z}(h_2,h_1)$ 
and  $\epsilon_{\mathbb P}(h_1)-\epsilon_{\mathbb P}(h_2)=\epsilon_{\mathbb P}(h_1,h_2)-\epsilon_{\mathbb P}(h_2,h_1)$. Then, adding \eqref{eq:20} and  \eqref{eq:21} gives us the proof.

\section*{Appendix B: proof of Theorem~2}
First we show that if the inequalities in Corollary 1 hold simultaneously for all $k\ge 1$ , we have $h^*\in \mathcal H_k$ for all $k\ge 1$. This can be proved by induction as following: First, clearly $h^*\in\mathcal H_1$. Assume $h^*\in\mathcal H_k$, apply the inequality in Corollary 1 with $h=h^*$ we have
\begin{align}
\epsilon_{\mathbb P|\mathcal D_k}( h_k^*)-\epsilon_{\mathbb P|\mathcal D_k}(h^*) \le \epsilon_{\mathcal Z_k}( h_k^*)-\epsilon_{\mathcal Z_k}(h^*)+\Delta_{\mathcal Z_k}(h,h_k^*).
\end{align}
Note that $\epsilon_{\mathbb P|\mathcal D_k}(h^*)\le \epsilon_{\mathbb P|\mathcal D_k}( h_k^*)$. Therefore,
\begin{align}
\epsilon_{\mathcal Z_k}(h^*)-\epsilon_{\mathcal Z_k}( h_k^*) & \leq \Delta_{\mathcal Z_k}(h,h_k^*)+\epsilon_{\mathbb P|\mathcal D_k}(h^*)-\epsilon_{\mathbb P|\mathcal D_k}( h_k^*)\\
& \leq \Delta_{\mathcal Z_k}(h,h_k^*).
\end{align}
This indicates that $h^*\in \mathcal H_{k+1}$ by the querying rule of the online active learning algorithm. By Corollary 1, the inequalities hold simultaneously with probability at least $1-\frac {1}{2T}$. Therefore we have $\Pr(h^*\in\mathcal H_{k},\forall k\ge 1)\ge 1-\frac {1}{2T}$. By the labeling rule of the online active learning algorithm, $h^*\in\mathcal H_{k},\forall k\ge 1$ implies $R(T)=0$. Hence, $\Pr(R(T)=0)\ge 1-\frac {1}{2T}$.
Hence,
\begin{equation}
\mathbb E[R(T)]  \le \Pr(R(T)>0)\cdot T= \frac 1{2T} \cdot T= 1/2.
\end{equation}
as desired.
\vspace{-0.4cm}
\section*{Appendix C: proof of Theorem~3}
We separate the analysis into two stages. Let $k_t$ be epoch index at time $t$ and $\tau= \min \{ t:\phi(\mathcal D_{k_t}) < T^{-\frac{\alpha}{2-\alpha}}\}-1$ (let $\tau=T$ if there is no such $t$). We first bound the label complexity of the first stage $\mathbb E[Q(\tau)]$. Let
\begin{equation}
c=18^\alpha \theta c_0 m^{-\frac \alpha 2} < 1.
\end{equation}
By the definition of $\beta_M$ we can show that 
\begin{equation}
\frac{18}{\sqrt{m T^{\frac{2-2\alpha}{2-\alpha}}}}\ge 3\beta_M^2+3\sqrt 2\beta_{M}.
\end{equation}
Next we show that for all $k\le k_\tau$
\begin{equation}
\mathbb E[\phi(\mathcal D_{k+1})|\phi(\mathcal D_k)]\le  \left(\frac{1+c}{2} \right) \phi(\mathcal D_k).
\end{equation}
Let
\begin{equation}
\mathcal H_k^{\theta}= \left\{h\in\mathcal H_k,\rho(h,h^*)>\frac{c\phi(\mathcal D_k)}{\theta} \right\}.
\end{equation}
If $h\in\mathcal H_k^{\theta}$, then
\begin{equation}
\epsilon_{\mathbb P|{\mathcal D_k}}(h)-\epsilon_{\mathbb P|{\mathcal D_k}}(h^*)= \frac{\epsilon_{\mathbb P}(h)-\epsilon_{\mathbb P}(h^*)}{\phi(\mathcal D_k)}\ge\left (\frac{\rho(h,h^*)}{c_0}\right)^{\frac 1\alpha} \frac{1}{\phi(\mathcal D_k)}.
\end{equation}
Since $\phi(\mathcal D_{k}) < T^{-\frac{\alpha}{2-\alpha}}$ for all $k\le k_\tau$,  we have
\begin{equation}
\left (\frac{\rho(h,h^*)}{c_0}\right)^{\frac 1\alpha} \frac{1}{\phi(\mathcal D_k)}= \frac{18\phi(\mathcal D_k)^{\frac 1\alpha }}{\sqrt{m}}\frac{1}{\phi(\mathcal D_k)} \ge\frac{18}{\sqrt{m T^{\frac{2-2\alpha}{2-\alpha}}}}.
\end{equation}
 Thus by Corollary 1, we can conclude that
\begin{align*}
	\epsilon_{\cZ_k}(h) - \epsilon_{\cZ_k}(h_k^*) & \geq \epsilon_{\mathbb P|\cD_k}(h) - \epsilon_{\mathbb P|\cD_k}(h_k^*) - \beta_M^2 - \beta_M \left(\sqrt{\epsilon_{\cZ_k}(h, h_k^*)} + \sqrt{\epsilon_{\cZ_k}(h_k^*, h)} \right) \\
	& \geq \epsilon_{\mathbb P|\cD_k}(h) - \epsilon_{\mathbb P|\cD_k}(h^*) + \epsilon_{\mathbb P|\cD_k}(h^*) - \epsilon_{\mathbb P|\cD_k}(h_k^*) - \beta_M^2 - \sqrt{2}\beta_M  \\
	& > \frac{18}{\sqrt{mT^{\frac{2-2\alpha}{2-\alpha}}}} - \Delta_{\cZ_k}(h^*, h_k^*) - \beta_M^2 - \sqrt{2}\beta_M  \\
	& > \frac{18}{\sqrt{mT^{\frac{2-2\alpha}{2-\alpha}}}} -  2\beta_M^2 - 2\sqrt{2}\beta_M  \\
	& > \beta_M^2 + \sqrt{2}\beta_M  \\
	& > \Delta_{\cZ_k}(h, h_k^*)
\end{align*}
with probability $1-\frac 1{2T}$. This indicates that for all $h\in\mathcal H_k^\theta $, $h\notin \mathcal H_{k+1}$. By the definition of $\theta$, we have $\phi(\mathcal D_{k+1})\le\phi(\Psi(\mathcal H_k\setminus\mathcal H_k^{\theta}))\le\dfrac{c\phi(\mathcal D_k)}{\theta}\cdot \theta=c\phi(\mathcal D_k)$ with probability at least $1-\frac 1{2T}$. Therefore $\mathbb E[\phi(\mathcal D_{k+1})|\phi(\mathcal D_k)]\le  (\frac{1+c}{2}) \phi(\mathcal D_k)$ as desired. Furthermore, we have
\begin{equation}
\mathbb E[\phi(\mathcal D_{k})]\le \left(\frac{1+c}{2}\right)^k \phi(\mathcal D_0)= \left(\frac{1+c}{2}\right)^k .
\end{equation}

Define $S(t)=(\frac 2 {1+c})^{\frac {Q(t)}{M}}-(\frac 2 {1+c})[(\frac 2 {1+c})^{\frac 1M}-1]t$. Next we show that $S_t$ is a supermartingale.  Let $Q(t)$ denote the label complexity at time $t$. Since $Q(t)\le (k_t+1)M$, we have
\begin{align}
\Pr \left(Q(t+1)=Q(t)+1|S(1),S(2),\cdots,S(t)\right) & = \Pr (q_{t+1}=1|S(1),S(2),\cdots,S(t)) \nonumber \\
& = \mathbb E[\phi(\mathcal D_{k_t})|Q(t)] \nonumber \\
& = \mathbb E\left[\phi \left(\mathcal D_{\lfloor \frac{Q(t)}{M}\rfloor} \right)\right] \nonumber\\
& \leq \left(\frac{1+c}{2}\right)^{ \frac{Q(t)}{M}-1}.
\end{align}
Therefore,
\begin{equation}
\begin{aligned}
&\mathbb E[S(t+1)|S_1,\ldots,S(t)]\\
=&\mathbb E\left[ \left(\frac 2 {1+c}\right)^{\frac{Q(t)}{M}}-\left(\frac 2 {1+c}\right)\left(\left(\frac 2 {1+c}\right)^{\frac 1M}-1\right)t \bigg| S_1,\ldots,S(t)\right]\\
\le&  \left(\frac 2 {1+c}\right)^{\frac{Q(t)+1}{M}}\left(\frac{1+c}2\right)^{\frac{Q(t)}{M}-1} +\left(\frac 2 {1+c}\right)^{\frac{Q(t)}{M}}\left(1-\left(\frac{1+c}2\right)^{\frac{Q(1)}{M}-1}\right)\\
&-\left(\frac 2 {1+c}\right)\left[\left(\frac 2 {1+c}\right)^{\frac 1M}-1\right](t+1)\\
=&\left(\frac 2 {1+c}\right)^{\frac {Q(t)}{M}}-\left(\frac 2 {1+c}\right)\left[\left(\frac 2 {1+c}\right)^{\frac 1M}-1\right]t=S(t)
\end{aligned}
\end{equation}
as desired. Then by optional stopping theorem,
\begin{equation}
\begin{aligned}
&\mathbb E\left[\left(\frac 2 {1+c}\right)^{\frac {Q(\tau)}{M}}-\left(\frac 2 {1+c}\right)\left(\left(\frac 2 {1+c}\right)^{\frac 1M}-1\right)\tau\right]\\
=&\mathbb E[S(\tau)]\le \mathbb E[S(0)]=1.
\end{aligned}
\end{equation}
Since $\tau\le T$, we have
\begin{equation}
\mathbb{E}\left[\left(\frac 2 {1+c}\right)^{\frac {Q(\tau)}{M}}\right]\le \left(\frac 2 {1+c}\right)\left(\left(\frac 2 {1+c}\right)^{\frac 1M}-1\right)T+1
\end{equation}
Since $f(x)=\log x$ is concave, by Jensen's Inequality,
\begin{equation}
\mathbb E[Q(\tau)]\le M\log_{(\frac 2 {1+c})} \left[\left(\frac 2 {1+c}\right)\left(\left(\frac 2 {1+c}\right)^{\frac 1M}-1\right)T+1\right].
\end{equation}
Since $M\le mdT^{\frac{2-2\alpha}{2-\alpha}}\log T+1$, $c<1$ and $(\frac 2 {1+c})((\frac 2 {1+c})^{\frac 1M}-1)<2$, we have
\begin{equation}
\mathbb E[Q(\tau)]\le \frac{2mdT^{\frac{2-2\alpha}{2-\alpha}}}{\log \frac{2}{1+c}} (\log T+1)^2=O(dT^{\frac{2-2\alpha}{2-\alpha}}\log ^2T)
\end{equation}
%
%
By definition of $\tau$ we have $q_t \le T^{-\frac{\alpha}{2-\alpha}} , \forall t>\tau$ . Therefore
\begin{equation}
\begin{aligned}
\mathbb E[Q(T)] &= \mathbb E[Q(\tau)] + \mathbb{E}\left[\sum_{t=\tau+1} q_t  |\tau\right] \\
&\le \frac{2mdT^{\frac{2-2\alpha}{2-\alpha}}}{\log \frac{2}{1+c}} (\log T+1)^2 + T^{-\frac{\alpha}{2-\alpha}} \cdot T \\
&\le  \left(\frac{2md}{\log \frac{2}{1+c}} +1\right)T^{\frac{2-2\alpha}{2-\alpha}}(\log T+1)^2
\end{aligned}
\end{equation}
as desired.

\section*{Appendix D: Proof of Theorem~\ref{thm:lower_bound}}


For ease of understanding we present the main proof for the case when the instance space, $\cX$, is the interval $[0,1]$ and the hypothesis space, $\cH$, is the class of threshold classifiers. The extension to general hypothesis spaces is straightforward and is described at the end of the proof. We begin the proof by considering the expression for regret for a (possibly randomized) policy $\pi$,
\begin{align}
	\mathbb{E}[R_{\pi}(T)] & = \mathbb{E} \left[ \sum_{t = 1}^T \1 \{q_t = 0\}(\1 \{\lambda_t \neq Y_t \} - \1 \{h^{\ast}(X_t) \neq Y_t \}) \right] \nonumber \\
	 & = \mathbb{E} \left[ \sum_{t = 1}^T \1 \{q_t = 0\} \mathbb{E}_{Y_t} \left[(\1 \{\lambda_t \neq Y_t \} - \1 \{h^{\ast}(X_t) \neq Y_t \})\right] \right] \nonumber\\
	 & = \mathbb{E}\left[ \sum_{t = 1}^T \1 \{q_t = 0\} 2 \left| \eta(X_t) - \frac{1}{2} \right| \1 \{\lambda_t \neq h^{\ast}(X_t) \} \right]  \nonumber\\
	 & =  2 \sum_{t = 1}^T \mathbb{E}\left[ \left| \eta(X_t) - \frac{1}{2} \right| \Pr(q_t = 0| I_{t-1} ) \Pr (\lambda_t \neq h^{\ast}(X_t) | I_{t-1}, q_t = 0) \right] \label{eq:lb_proof_1}
\end{align}
where $I_t$ denotes the information vector up to time $t$. It contains all the information obtained up to and including time $t$ in terms of observed instances and their labels along with any additional information pertaining to the algorithm up to time $t$. We now focus on bounding the probability of making an error. 

To bound the probability of making a mistake, we restrict ourselves to a subproblem which is easier to analyze. Fix a $\mu > 0$ and consider a pair of hypotheses $h_1, h_2$ such that $\rho(h_1, h^*) = \mu$, $\rho(h_2, h^*) = \mu/2$ and $\rho(h_1, h_2) = \mu/2$. For example, when the distribution is uniform and if a hypothesis is represented by a point in the interval, then a possible option is $h_1 = h^* - \mu$ and $h_2 = h^* - \mu/2$. For the analysis, we would only consider the regret incurred by $\pi$ in the region $\DIS(h_1, h_2)$ where $\DIS(h, h') = \{ x : h(x) \neq h'(x) \}$. This is a subset of the instance space and consequently the total regret incurred by $\pi$ would be at least the regret it incurs on the instances in this region. 

To lower bound the probability of error when policy $\pi$ labels a point, we consider the subproblem of distinguishing between $h_1$ and $h^*$. More specifically, given certain number of labeled instances in $\DIS(h_1, h^*)$ we bound the probability that \emph{any} randomized policy would be unable to identify the correct hypothesis between $h_1$ and $h^*$ using the labeled instances. The motivation is that if $\pi$ incorrectly concludes that $h_1$ is the true hypothesis and proceeds to label a point in $\DIS(h_1, h_2)$ then such a labeling event would contribute to regret. And since $\pi$ is a policy that performs uniformly well for all $\mathbb{P}_{Y|X}$, it would have to distinguish between $h_1$ and $h^*$ during the course of learning and the therefore the regret incurred by the policy is at least as much as it incurs in trying to distinguish these two hypotheses.  

Since we are focusing on the binary hypothesis problem of distinguishing between $h_1$ and $h^*$, we can restrict ourselves to the instances observed in $\DIS(h_1, h^*)$. Recall that we are considering threshold classifiers and therefore in the disagreement region between two hypotheses, one of them will label all the points to be $+1$ while the other would label all of them to be $-1$. Consequently, the problem of identifying the correct hypothesis between $h_1$ and $h^*$ is equivalent to the problem of identifying the parameter of a Bernoulli random variable. The following technical lemma from Anthony and Bartlett~\cite{anthony_bartlett_1999}, which provides a lower bound on error in estimating the parameter of a Bernoulli random variable, would be useful in further analysis.
\begin{lemma}
	Suppose that $\alpha$ is a random variable that is uniformly distributed on $\{\alpha_{+}, \alpha_{-}\}$ where $\alpha_{-} = \dfrac{1}{2} - \dfrac{\gamma}{2}$ and $\alpha_{+} = \dfrac{1}{2} +\dfrac{\gamma}{2}$ with $\gamma	\in (0,1)$. Suppose that $(\xi_1, \xi_2, \dots, \xi_m)$ are i.i.d. $\{0,1\}$ random variables with $\Pr(\xi_1 = 1) = \alpha$ for all $i$. Let $f$ be a function, possibly randomized, from $\{0,1\}^m \to \{\alpha_{+}, \alpha_{-}\}$, then we have 
	\begin{align*}
		\Pr(f(\xi_1, \xi_2, \dots, \xi_m) \neq \alpha) & > \frac{1}{4} \left( 1 - \sqrt{1 - \exp\left( \frac{-2\lceil m/2 \rceil \gamma^2}{1 - \gamma^2}\right)} \right) \\
		& > \frac{1}{8} \exp\left( \frac{-2\lceil m/2 \rceil \gamma^2}{1 - \gamma^2}\right)
	\end{align*}	 
\end{lemma}

We cannot directly apply lemma since our setup is slightly different from the one in the lemma. It is not difficult to see that in our case, the labels observed are independent but not identically distributed. However, it is straightforward to slightly tweak the proof of the above lemma to incorporate this condition. Let $m_t$ denote the number of instances queried by the policy $\pi$ up to (but not including) time instant $t$ in $\DIS(h_1, h^*)$ corresponding to points $X_1, X_2, \dots, X_{m_t}$. WLOG we can assume that $h^*$ labels these points as $-1$. The proof of the above lemma argues that probability of error is at least the probability of observing greater than $m_t/2$ points labeled $+1$. The probability of this event is more than that of observing greater than $m_t/2$ instances of $+1$ in an i.i.d. sample of size $m_t$ of a Bernoulli random variable $Z$, where $\Pr(Z = +1) = 1/2 - \gamma_{\max}/2$ where $\gamma_{\max} = 2 \max_{x \in \DIS(h_1, h^*)} |\eta(x) -  1/2|$. This follows from the fact that $\eta(x) \geq 1/2 - \gamma_{\max}/2$ for all $x \in \DIS(h_1, h^*)$. Therefore, we can lower bound the probability of interest by bounding the probability of error in estimating the parameter for $Z$ for which we can directly use the lemma. Using the result from the lemma and plugging it back into~\eqref{eq:lb_proof_1}, we get, 
\begin{align}
	\mathbb{E}[R_{\pi}(T)] & \geq 2 \sum_{t = 1}^T \mathbb{E}\left[ \left| \eta(X_t) - \frac{1}{2} \right| \Pr(q_t = 0| I_{t-1} ) \1 \{ X_t \in \DIS(h_1, h_2) \} \exp\left( \frac{-2\lceil m_t/2 \rceil \gamma_{\max}^2}{1 - \gamma_{\max}^2}\right)\right] \label{eq:lb_proof_2}
\end{align}
From the assumptions on the noise model, we can conclude that $\displaystyle \min_{x \in \DIS(h_1, h_2)} |\eta(x) - 1/2| \geq 0.5c_2' \mu^{1/\alpha - 1}$for some universal constant $c_2' > 0$. Also, if $Q(T)$ is random variable corresponding to the total number of queries by the policy $\pi$ until time $T$, then we have the trivial inequality $m_t \leq Q(T)$. Using these observations, we can rewrite~\eqref{eq:lb_proof_2} as,
\begin{align}
	\mathbb{E}[R_{\pi}(T)] & \geq \frac{c_2'}{8} \mu^{1/\alpha - 1} \mathbb{E} \left[ \sum_{t = 1}^T  \1 \{ q_t = 0, X_t \in \DIS(h_1, h_2) \} \exp\left( \frac{-2\lceil Q(T)/2 \rceil \gamma_{\max}^2}{1 - \gamma_{\max}^2}\right)\right] \label{eq:lb_proof_3}
\end{align}

To further bound the expression on RHS, we consider the number of instances that arrive in the $\DIS(h_1, h_2)$. Let $N$ be the random number of observed instances which belong to $\DIS(h_1, h_2)$. To bound the above expression we consider the event $N \geq \mu T/2$. From the results in~\cite{Greenberg2014}, we can conclude that the probability of such an event is at least $1/4$. To evaluate the expression involving $Q(T)$, we consider the event such that $Q(T) \leq \gamma_{\max}^{-2}$. Under this event, it is not difficult to note that the following holds $\displaystyle \exp\left( \frac{-2\lceil Q(T)/2 \rceil \gamma_{\max}^2}{1 - \gamma_{\max}^2}\right) \geq c_3$ for a universal constant $c_3 > 0$. Combining these two points with~\eqref{eq:lb_proof_3}, we get,
\begin{align}
	\mathbb{E}[R_{\pi}(T)] & \geq \frac{c_2'}{8} \mu^{1/\alpha - 1} \mathbb{E}\left[ (N - Q(T)) \exp\left( \frac{-2\lceil Q(T)/2 \rceil \gamma_{\max}^2}{1 - \gamma_{\max}^2}\right)\right] \\
	& \geq \frac{c_2'}{8} \mu^{1/\alpha - 1} \mathbb{E} \left[ \1 \{ N \geq \mu T/2\} \1 \{ Q(T) \leq \gamma_{\max}^{-2}\} (\rho T/2 - \gamma_{\max}^{-2}) c_3\right] \\
	& \geq \frac{c_3'}{32} \mu^{1/\alpha - 1} (\mu T/2 - \gamma_{\max}^{-2})  \Pr ( Q(T) \leq \gamma_{\max}^{-2}) 
\end{align}

To bound the probability on RHS, we use the constraint that the regret incurred by $\pi$ has to be bounded.  Notice that if the expression $\mu^{1/\alpha - 1} (\mu T/2 - \gamma_{\max}^{-2})$ is an increasing function of $T$, that is if $\mu^{1/\alpha - 1} (\mu T/2 - \gamma_{\max}^{-2}) \sim T^{\varepsilon}$ for some $\varepsilon > 0$, then we would have $\Pr ( Q(T) \leq \gamma_{\max}^{-2}) \sim T^{-\varepsilon}$ since the regret incurred by $\pi$ is bounded by a constant. More generally, if the allowed regret budget of policy $\pi$, denoted by $R_0(T)$, is sublinear function of $\mu^{1/\alpha - 1} (\mu T/2 - \gamma_{\max}^{-2})$, then $\Pr (Q(T) \leq \gamma_{\max}^{-2}) \lesssim T^{-\varepsilon}$ for some $\varepsilon > 0$. This implies that the probability that such a policy queries less than $\gamma_{\max}^{-2}$ samples would be small and using Markov's inequality we can conclude that $\mathbb{E}[Q(T)]$ would be $\Omega(\gamma_{\max}^{-2})$. From the assumption on the noise model we have $\gamma_{\max} \leq c_1 \mu^{\frac{1}{\alpha} - 1}$. Therefore, to obtain the tightest lower bound on expected label complexity, we find the smallest value of $\mu$ which ensures that $\mu^{1/\alpha - 1} (\mu T/2 - c_1 \mu^{2- \frac{2}{\alpha}})$ is an increasing function of $T$. On solving the equation, we get $\mu \sim T^{-\frac{\alpha}{2 - \alpha}}$ and consequently, $\mathbb{E}[Q(T)] \geq \Omega(T^{\frac{2 - 2\alpha}{2 - \alpha}})$ as required. To evaluate the lower bound for different regret budget $R_0(T)$, we just need to evaluate the smallest $\mu$ that would allow $R_0(T)$ to be sublinear in $\mu^{1/\alpha - 1} (\mu T/2 - c_1 \mu^{2- \frac{2}{\alpha}})$ and then the corresponding $\gamma_{\max}$ to obtain the bound.

It is not difficult to see that the assumption of $\cH$ being the class of threshold classifiers was not crucial to the fundamental idea of the proof and it was primarily taken to simplify the description of the hypothesis pair $h_1$ and $h_2$. The proof can be extended to general hypothesis classes by considering $h_1$ and $h_2$ such that $\Pr(\DIS(h_1, h^*)) = \mu$, $\Pr(\DIS(h_2, h^*)) = c\mu$ for $c < 1$ and given $\mu > 0$ with $\DIS(h_2, h^*) \subseteq \DIS(h_1, h^*)$. Additionally, $h_1$ and $h_2$ also satisfy $\Pr(\{x : h^*(x) = v \ \land \ h_1(x) = -v \}) = c_1 \mu$ and $\Pr(\{x : h^*(x) = v \ \land \ h_2(x) = -v \}) = c_2 \mu$ for $c_1 > c_2 > 0$ and some $v \in \{-1, 1 \}$. All the arguments in the proof follow exactly for a pair of hypotheses $h_1$ and $h_2$ that are chosen in aforementioned manner and consequently, the lower bound $\mathbb{E}[Q(T)] \geq \Omega(T^{\frac{2 - 2\alpha}{2 - \alpha}})$ holds for general hypothesis classes.

For the case of Massart noise, the proof follows a similar idea. From the condition on the Massart Noise, we have, $| \eta(x) - 1/2 | \geq \gamma_0 /2$ for all $x \in \mathcal{X}$. Since the slope, i.e. $| \eta(x) - 1/2 |$, ``jumps'' at the boundary, we do not need to consider $h_2$ at all. We just consider $h_1$ and $h^*$. Let $h_1$ be such that $\DIS(h_1, h^*) = \mu$ for some $\mu > 0$ and $\gamma_{\max} < 1$ where $\gamma_{\max} = 2 \max_{x \in \DIS(h_1, h^*)} |\eta(x) -  1/2|$, both of which are independent of $T$. Therefore, we can rewrite~\eqref{eq:lb_proof_2} as, 
\begin{align}
	\mathbb{E}[R_{\pi}(T)] & \geq 2 \sum_{t = 1}^T \mathbb{E}\left[ \left| \eta(X_t) - \frac{1}{2} \right| \Pr(q_t = 0| I_{t-1} ) \1 \{ X_t \in \DIS(h_1, h^*) \} \exp\left( \frac{-2\lceil m_t/2 \rceil \gamma_{\max}^2}{1 - \gamma_{\max}^2}\right)\right] \nonumber\\
	& \geq \gamma_0 \sum_{t = 1}^T \mathbb{E}\left[ \Pr(q_t = 0| I_{t-1} ) \1 \{ X_t \in \DIS(h_1, h^*) \} \exp\left( \frac{-2\lceil m_t/2 \rceil \gamma_{\max}^2}{1 - \gamma_{\max}^2}\right)\right] .
\end{align}
Again, letting $N$ denote the number of instances observed in $\DIS(h_1, h^*)$ and $Q(T)$ being the query complexity, we can rewrite the above equation as
\begin{align}
	\mathbb{E}[R_{\pi}(T)] & \geq \gamma_0 \sum_{t = 1}^T \mathbb{E}\left[ (N - Q(T)) \exp\left( \frac{-2\lceil Q(T)/2 \rceil \gamma_{\max}^2}{1 - \gamma_{\max}^2}\right)\right],
\end{align}
where we use the trivial inequality $m_t \leq Q(T)$. Considering the event $N \geq \mu T$, we can write the above equation as
\begin{align}
	\mathbb{E}[R_{\pi}(T)] & \geq \frac{\gamma_0}{4} \frac{- 2 \gamma_{\max}^2}{1 - \gamma_{\max}^2} \mathbb{E}\left[ \mu T- Q(T) \exp\left( \frac{- Q(T) \gamma_{\max}^2}{1 - \gamma_{\max}^2}\right)\right].
\end{align}
Using Jensen inequality and noting that $xe^{-ax} \leq (ae)^{-1}$ for all $x \geq 0$, we can rearrange the above equation as
\begin{align}
	\frac{\mathbb{E}[R_{\pi}(T)]}{c_{\gamma}} + \frac{1 - \gamma_{\max}^2}{e \gamma_{\max}^2} & \geq  T \exp\left( \frac{- \mathbb{E}[Q(T)] \gamma_{\max}^2}{1 - \gamma_{\max}^2}\right),
\end{align}
where $\displaystyle c_{\gamma} = \frac{\gamma_0}{4} \frac{- 2 \gamma_{\max}^2}{1 - \gamma_{\max}^2}$. Noting that $\mathbb{E}[R_{\pi}(T)]$ is a sublinear function of $T$, i.e., $\mathbb{E}[R_{\pi}(T)]$ is $O(T^{\varepsilon})$ for some $\varepsilon \in (0,1)$, we can rearrange the above equation to obtain $\mathbb{E}[Q(T)]$ is $\Omega(\log T)$, as required.

\section*{Appendix E: proof of Theorem~\ref{thm:lc_RW_OLA}}
We first introduce the convergence VC bound~\cite{vapnik2015uniform} to establish the relationship between the empirical error and true error rate of any hypothesis $h$. 
\begin{lemma}~\cite{vapnik2015uniform}
\label{lemma:vc} 
Let $\mathcal Z$ bet a set of $M$ i.i.d. (X,Y)-samples under distribution $\mathbb P$. For all $h\in \mathcal H$, we have, with probability at least $1-\delta$,
\begin{equation}
|\epsilon_P(h)  - \epsilon_{\mathcal Z}(h)| \le \Delta(M,\delta)
\end{equation}
\end{lemma}
%
%
%
Define
$$
d(\mathcal H_k)=\min_{h\notin \mathcal H_k} \epsilon_{\mathbb P_{\mathcal D_{r(k)}}}(h) - \min_{h\in \mathcal H_k} \epsilon_{\mathbb P_{\mathcal D_{r(k)}}}(h)
$$
which captured the hardness of the hypothesis testing problem.

\begin{definition}
Define a random process $W_k$ for each epoch $k$ as following: 
\begin{enumerate}
\item Let $W_1 = 0$
\item If $d(\mathcal H_k)\ge 4\Delta(M,1-\sqrt p)$:
\begin{enumerate}
\item If the verification passed with $d(\mathcal H_{k+1})\ge 4\Delta(M,1-\sqrt p)$, let $W_{k+1}=W_k+1$
\item If the verification passed with $d(\mathcal H_{k+1})< 4\Delta(M,1-\sqrt p)$,  let $W_{k+1}=W_k-1$
\item If the verification failed, let $W_{k+1} = W_k-1$
\end{enumerate}
\item If $0\le d(\mathcal H_k)<4\Delta(M,1-\sqrt p)$:
\begin{enumerate}
\item If the verification passed with $d(\mathcal H_{k+1})\ge 4\Delta(M,1-\sqrt p)$, let $W_{k+1}=W_k+1$
\item If the verification passed with $d(\mathcal H_{k+1})< 4\Delta(M,1-\sqrt p)$,  let $W_{k+1}=W_k-1$
\item If the verification failed, let $W_{k+1} = W_k+1$
\end{enumerate}
\item If $d(\mathcal H_k) <0$:
\begin{enumerate}
\item If the verification passed, let $W_{k+1}=W_k-1$
\item If the verification failed, let $W_{k+1} = W_k+1$
\end{enumerate}
\end{enumerate}
\end{definition}

\begin{lemma}
For all epoch $k$, we have
$$
\Pr(W_{k+1} = W_k + 1) \ge p
$$
\end{lemma}

\begin{proof}  

Based on how $W_{k+1}$ is defined, we consider following three cases:

\begin{enumerate}
\item If $d(\mathcal H_k)\ge 4\Delta(M,1-\sqrt p)$:

First we show that the probability that the algorithm zooms in is at least $p$.
By definition of $d(\mathcal H_k)$, we have $\min_{h\notin \mathcal H_k} \epsilon_{\mathbb P_{\mathcal D_{r(k)}}}(h) - \min_{h\in \mathcal H_k} \epsilon_{\mathbb P_{\mathcal D_{r(k)}}}(h)\ge 4\Delta(M,1-\sqrt p)$. By Lemma~\ref{lemma:vc}, with probability at least $\sqrt p$, we have
\begin{equation}
|\epsilon_P(h)  - \epsilon_{\mathcal Z}(h)| \le \Delta(M,1-\sqrt p)
\end{equation}
for all $h\in \mathcal H$. Therefore,
\begin{equation}
\min_{h\notin\mathcal H_k} \epsilon_{\mathcal Z_k'} (h) \ge \min_{h\notin \mathcal H_k} \epsilon_{\mathbb P_{\mathcal D_{r(k)}}}(h) - \Delta(M,1-\sqrt p)
\end{equation}
and
\begin{equation}
\min_{h\in\mathcal H_k} \epsilon_{\mathcal Z_k'} (h) \le \min_{h\in \mathcal H_k} \epsilon_{\mathbb P_{\mathcal D_{r(k)}}}(h) + \Delta(M,1-\sqrt p)
\end{equation}
Hence
\begin{equation}
\min_{h\notin\mathcal H_k} \epsilon_{\mathcal Z_k'} (h) - \min_{h\in\mathcal H_k} \epsilon_{\mathcal Z_k'} (h) \ge 4\Delta(M,1-\sqrt p) - 2\Delta(M,1-\sqrt p) = 2\Delta(M,1-\sqrt p)
\end{equation}
Therefore, the probability that the algorithm zooms in is at least $\sqrt p$ as desired. Then, by applying Lemma~\ref{lemma:vc} again, we have
\begin{equation}
\begin{aligned}
\min_{h\notin \mathcal H_{k+1}}\epsilon_{\mathbb P_{\mathcal D_{k}}} (h) \ge &\min_{h\notin \mathcal H_{k+1}}\epsilon_{\mathcal Z_k} (h) -\Delta(M,1-\sqrt p) \\
\ge& \epsilon_{{\mathcal Z_{k}}} (h_k^*) +6\Delta(M,1-\sqrt p)  - \Delta(M,1-\sqrt p) \\
\ge & \epsilon_{\mathbb P_{\mathcal D_{k}}} (h_k^*) - \Delta(M,1-\sqrt p) + 5\Delta(M,1-\sqrt p) \\
\ge & \epsilon_{\mathbb P_{\mathcal D_{k}}} (h^*) + 4\Delta(M,1-\sqrt p)
\end{aligned}
\end{equation}
with probability at least $\sqrt p$. Therefore, $\Pr(W_{k+1}=W_k+1|d(\mathcal H_k)\ge 4\Delta(M,1-\sqrt p)) \ge (\sqrt p)^2= p$.

\item If $0\le d(\mathcal H_k)<4\Delta(M,1-\sqrt p)$:

Since $W_{k+1} = W_k+1$ when the algorithm zooms out, it suffices to show that when the algorithm zooms in, $d(\mathcal H_{k+1})\ge 4\Delta(M,1-\sqrt p)$ with probability at least $\sqrt p$. By applying Lemma~\ref{lemma:vc}, we have
\begin{equation}
\begin{aligned}
\min_{h\notin \mathcal H_{k+1}}\epsilon_{\mathbb P_{\mathcal D_{k}}} (h) \ge &\min_{h\notin \mathcal H_{k+1}}\epsilon_{\mathcal Z_k} (h) -\Delta(M,1-\sqrt p) \\
\ge& \epsilon_{\mathbb P_{\mathcal Z_{k}}} (h_k^*) +6\Delta(M,1-\sqrt p)  - \Delta(M,1-\sqrt p) \\
\ge & \epsilon_{\mathbb P_{\mathcal D_{k}}} (h_k^*) - \Delta(M,1-\sqrt p) + 5\Delta(M,1-\sqrt p) \\
\ge & \epsilon_{\mathbb P_{\mathcal D_{k}}} (h^*) + 4\Delta(M,1-\sqrt p)
\end{aligned}
\end{equation}
with probability at least $\sqrt p$.
Therefore, $\Pr(W_{k+1}=W_k+1|0\le d(\mathcal H_k)<4\Delta(M,1-\sqrt p)) \ge \sqrt p \ge p$.

\item If $d(\mathcal H_k)<0$:

It suffices to show that the probability that the algorithm zooms out is at least $\sqrt p$. By definition, we have
\begin{equation}
\min_{h\notin \mathcal H_k} \epsilon_{\mathbb P_{\mathcal D_{r(k)}}}(h) < \min_{h\in \mathcal H_k} \epsilon_{\mathbb P_{\mathcal D_{r(k)}}}(h)
\end{equation}
Therefore, by applying Lemma~\ref{lemma:vc}, with probability at least $\sqrt p$, we have
\begin{equation}
\min_{h\notin\mathcal H_k} \epsilon_{\mathcal Z_k'} (h) \le \min_{h\notin \mathcal H_k} \epsilon_{\mathbb P_{\mathcal D_{r(k)}}}(h) + \Delta(n,1-p)
\end{equation}
and
\begin{equation}
\min_{h\in\mathcal H_k} \epsilon_{\mathcal Z_k'} (h) \ge \min_{h\in \mathcal H_k} \epsilon_{\mathbb P_{\mathcal D_{r(k)}}}(h) - \Delta(n,1-p).
\end{equation}
Consequently, we have
\begin{equation}
\min_{h\notin\mathcal H_k} \epsilon_{\mathcal Z_k'} (h) - \min_{h\in\mathcal H_k} \epsilon_{\mathcal Z_k'} (h) \le  2\Delta(M,1-\sqrt p).
\end{equation}
Therefore, $\Pr(W_{k+1}=W_k+1|0\le d(\mathcal H_k)<0) \ge \sqrt p \ge p$
\end{enumerate}
Based on the three cases considered, we have $\Pr(W_{k+1}=W_k+1) \ge p$ as desired.
\end{proof}

Since either $W_{k+1} = W_k+1$ or $W_{k+1}=W_k-1$, $W_k$ is a random walk process with positive bias at least $p$.

\begin{lemma}
\label{lemma:condition}
There exists $c_2<1$ such that
$$
\mathbb E[\phi(\mathcal D_{r(k)}) |W_k=w] \le c_2^w
$$
\end{lemma}
\begin{proof}
First we show that $ E[\phi(\mathcal D_{k+1})| \phi(\mathcal D_{k}), r(k+1)=k, h^*\in\mathcal H_k] \le c \phi(\mathcal D_{k})$.
Let
\begin{equation}
c = 32 \theta c_0 m^{-\frac 1 2} < 1.
\end{equation}
Let
\begin{equation}
\mathcal H_k^{\theta}=\left\{h\in\mathcal H_k,\rho(h,h^*)>\frac{c\phi(\mathcal D_k)}{\theta}\right\}.
\end{equation}
By the definition of $\Delta(M,1-\sqrt p)$ we can show that 
\begin{equation}
\frac{32}{\sqrt{m }}\ge 8\Delta(M,1-\sqrt p).
\end{equation}
If $h\in\mathcal H_k^{\theta}$, then
\begin{equation}
\epsilon_{\mathbb P|{\mathcal D_k}}(h)-\epsilon_{\mathbb P|{\mathcal D_k}}(h^*)= \frac{\epsilon_{\mathbb P}(h)-\epsilon_{\mathbb P}(h^*)}{\phi(\mathcal D_k)}\ge\frac{\rho(h,h^*)}{c_0\phi(\mathcal D_k)}.
\end{equation}
Consequently, we have
\begin{equation}
 \frac{\rho(h,h^*)}{c_0} \frac{1}{\phi(\mathcal D_k)}= \frac{32\phi(\mathcal D_k)}{\sqrt{m}}\frac{1}{\phi(\mathcal D_k)} \ge\frac{32}{\sqrt{m }}.
\end{equation}
 Thus by Lemma~\ref{lemma:vc}, we can conclude that
\begin{align*}
	\epsilon_{\cZ_k}(h) - \epsilon_{\cZ_k}(h_k^*) & \geq \epsilon_{\cZ_k}(h) - \epsilon_{\cZ_k}(h^*) \\
& \geq \epsilon_{\mathbb P|\cD_k}(h) - \epsilon_{\mathbb P|\cD_k}(h^*) - 2\Delta(M,1-\sqrt p) \\
	& > \frac{32}{\sqrt{m }}  - 2\Delta(M,1-\sqrt p)  \\
	& > 6\Delta(M,1-\sqrt p)  
\end{align*}
with probability $p$. This indicates that for all $h\in\mathcal H_k^\theta $, $h\notin \mathcal H_{k+1}$ when zoomed in. By the definition of $\theta$, we have $\phi(\mathcal D_{k+1})\le\phi(\Psi(\mathcal H_k\setminus\mathcal H_k^{\theta}))\le\dfrac{c\phi(\mathcal D_k)}{\theta}\cdot \theta=c\phi(\mathcal D_k)$ with probability at least $p$. Therefore $ E[\phi(\mathcal D_{k+1})| \phi(\mathcal D_{k}), r(k+1)=k] \le c_2 \phi(\mathcal D_{k})$ as desired, where $c_2 = c_0p+1-c_0<1$.

For each $k$, we can find a sequence of $(i_1,i_2,\cdots,i_{j_k},k)$ where $r(i_{l+1})=i_l$ for $l=1,2,\cdots,i_{j_k}$ and $r(k)=i_{j_k}$. Let $j^*(k)$ be the largest integer such that $h^*\in \mathcal H_{i_{j^*(k)}}$. Then, by the definition of $W_k$, we have $j^*(k)\ge W_k$. Note that $\mathcal H_{i_{l+1}} \subseteq \mathcal H_{i_{l}}$ for all $l=1,2,\cdots,i_{j_k}$, we can apply the inequality $\mathbb E[\phi(\mathcal D_{k+1})| \phi(\mathcal D_{k}), r(k+1)=k ,h^*\in\mathcal H_k] \le c_2\phi(\mathcal D_{k})$ $j^*(k)$ times, which gives us
$$
\mathbb E[\phi(\mathcal D_{r(k)}) |W_k=w] \le \mathbb E[\phi(\mathcal D_{i_{j^*(k)}}) |W_k=w]  \le c_2^{j^*(k)} \le c_2^w
$$
as desired.
\end{proof}

\begin{lemma}
There exists $c_3<1$ such that
\label{lemma:tower}
$$
\mathbb E[\phi(\mathcal D_{r(k)})] \le c_4\cdot c_3^k
$$
\end{lemma}

\begin{proof}
By tower property and Lemma~\ref{lemma:condition}, we have
\begin{equation}
\mathbb E[\phi(\mathcal D_{r(k)})] = \mathbb E[\mathbb E[\phi(\mathcal D_{r(k)})|W_k] ] = \mathbb E[c_2^{W_k}].
\end{equation}
Since $W_k$ is a random walk process with positive bias at least $p$. Let $V_k= Bin(k,p)$ be a binomial random variable. Then for any $n\ge 0$ we have $P(W_k\ge n) \ge P(2V_k-k\ge n)$. Therefore, by interchanging the sum order and using the Moment Generating Function for Binomial random variable, we have
\begin{equation}
\mathbb E[c_2^{W_k}] \le \mathbb E [c_2^{2V_k}]   =  \left(\frac{1-p+c_2^2p}{c_2}\right)^k = c_3^k
\end{equation}
as desired where $c_3=\dfrac{1-p+c_2^2p}{c_2}<1$.
\end{proof}
Let $Q(t)$ denote the label complexity at time $t$. Define $S(t)=(\frac 1 {c_3})^{\frac {Q(t)}{2M}}-(\frac {1} {c_3})[(\frac 1 {c_3})^{\frac 1 {2M}}-1]t$. Next we show that $S(t)$ is a supermartingale.  Since $Q(t)\le 2(k_t+1)M$, we have
\begin{align}
\Pr \left(Q(t+1)=Q(t)+1|S(1),S(2),\cdots,S(t)\right) & = \Pr (q_{t+1}=1|S(1),S(2),\cdots,S(t)) \nonumber \\
& = \mathbb E[\phi(\mathcal D_{r(k_t)})|Q(t)] \nonumber \\
& \le  c_3^{ \frac{Q(t)}{2M}-1}.
\end{align}
Therefore,
\begin{equation}
\begin{aligned}
&\mathbb E[S_{t+1}|S(1),S(2),\cdots,S(t)]\\
=& \ \mathbb{E} \left[ \left(\frac 1 {c_3}\right)^{\frac{Q(t)}{2M}}-\left(\frac 1 {c_3}\right)\left(\left(\frac 1 {c_3}\right)^{\frac 1{2M}}-1\right)t \bigg|S(1),S(2),\cdots,S(t)\right]\\
\le& \  \left(\frac 1 {c_3}\right)^{\frac{Q(t)+1}{2M}}\left(\frac{1+c}2\right)^{\frac{Q(t)}{2M}-1} +  \left(\frac 1 {c_3}\right)^{\frac{Q(t)}{2M}}\left(1-\left(\frac{1+c}2\right)^{\frac{Q(1)}{2M}-1}]\right)\\
& \ - \left(\frac 1 {c_3}\right)\left[\left(\frac 1 {c_3}\right)^{\frac 1 {2M}}-1\right](t+1)\\
=& \ \left(\frac 1 {c_3}\right)^{\frac {Q(t)}{2M}}-\left(\frac 1 {c_3}\right)\left[\left(\frac 1 {c_3}\right)^{\frac 1{2M}}-1\right]t=S(t)
\end{aligned}
\end{equation}
as desired. Then by optional stopping theorem,
\begin{equation}
\begin{aligned}
&\mathbb E\left[\left(\frac 1 {c_3}\right)^{\frac {Q(t)}{2M}}-\left(\frac 1 {c_3}\right)\left(\left(\frac 1 {c_3}\right)^{\frac 1{2M}}-1\right)t\right]\\
=& \ \mathbb E[S(T)]\le \ \mathbb E[S(0)] = 1.
\end{aligned}
\end{equation}
Hence,
\begin{equation}
\mathbb E\left[\left(\frac 1 {c_3}\right)^{\frac {Q(T)}{2M}}\right]\le \left(\frac 1 {c_3}\right)\left(\left(\frac 1 {c_3}\right)^{\frac 1{2M}}-1\right)T+1.
\end{equation}
Since $f(x)=\log x$ is concave, by Jensen's Inequality we have,
\begin{equation}
\mathbb E[Q(T)]\le M\log_{\left(\frac 1 {c_3}\right)} \left[\left(\frac 1 {c_3}\right)\left(\left(\frac 1 {c_3}\right)^{\frac 1{2M}}-1\right)T+1\right].
\end{equation}
Since $M\le md$, $c_3<\frac 12$ and $(\frac 1 {c_3})((\frac 1 {c_3})^{\frac 1{2M}}-1)<2$, we have
\begin{equation}
\mathbb E[Q(\tau)]\le \frac{2md}{\log \frac 1{c_3}} \log (2T+1)=O(d\log T)
\end{equation}
as desired.

\section*{Appendix F: proof of Theorem~\ref{thm:reg_RW_OLA}}
\begin{lemma}
There exists $c_5>0$ such that
$ E[\phi(\mathcal D_{k+1})| \phi(\mathcal D_{k})] \ge c_5 \phi(\mathcal D_{k})$.
\end{lemma}
\begin{proof}
Let
\begin{equation}
c'=8 \theta' c_0 m^{-\frac 1 2} 
\end{equation}
and
\begin{equation}
\mathcal V_k^{\theta'}= \left\{h\in\mathcal H_k,\rho(h,h^*)\le\frac{c'\phi(\mathcal D_k)}{\theta'} \right\}.
\end{equation}
By the definition of $\Delta(M,1-\sqrt p)$ we can show that 
\begin{equation}
\frac{8}{\sqrt{m }}\le 4\Delta(M,1-\sqrt p).
\end{equation}
Here we assume that there exists a constant $c_0'$ such that 
$\displaystyle c_0'(d(h, h^*)) \leq \rho(h, h^*)$ holds for all $h \in \mathcal{H}$. Note that the worst case we have $c_0'=1$. If $h\in\mathcal V_k^{\theta'}$, then
\begin{equation}
\epsilon_{\mathbb P|{\mathcal D_k}}(h)-\epsilon_{\mathbb P|{\mathcal D_k}}(h^*)= \frac{\epsilon_{\mathbb P}(h)-\epsilon_{\mathbb P}(h^*)}{\phi(\mathcal D_k)}\le\frac{\rho(h,h^*)}{c_0'\phi(\mathcal D_k)}.
\end{equation}
Using the above equation, we can write,
\begin{equation}
 \frac{\rho(h,h^*)}{c_0'} \frac{1}{\phi(\mathcal D_k)}= \frac{8\phi(\mathcal D_k)}{\sqrt{m}}\frac{1}{\phi(\mathcal D_k)} \le\frac{8}{\sqrt{m }}.
\end{equation}
 Thus by Lemma~\ref{lemma:vc}, we can conclude that
\begin{align*}
	\epsilon_{\cZ_k}(h) - \epsilon_{\cZ_k}(h_k^*) & \le \epsilon_{\cZ_k}(h) - \epsilon_{\cZ_k}(h^*) \\
& \le \epsilon_{\mathbb P|\cD_k}(h) - \epsilon_{\mathbb P|\cD_k}(h^*) +2\Delta(M,1-\sqrt p) \\
	& \le \frac{8}{\sqrt{m }}  + 2\Delta(M,1-\sqrt p)  \\
	& \le 6\Delta(M,1-\sqrt p)  
\end{align*}
with probability $p$. This indicates that for all $h\in\mathcal V_k^{\theta'} $, $h\in \mathcal H_{k+1}$ when zoomed in. By the definition of $\theta$, we have $\phi(\mathcal D_{k+1})\ge\phi(\Psi(\mathcal V_k^{\theta'}))\ge\dfrac{c'\phi(\mathcal D_k)}{\theta'}\cdot \theta'=c'\phi(\mathcal D_k)$ with probability at least $p$. Therefore $ E[\phi(\mathcal D_{k+1})| \phi(\mathcal D_{k})] \ge c_5 \phi(\mathcal D_{k})$ as desired, where $c_5 = c'p$.
\end{proof}

\begin{definition}
Define a random process $U_k$ for each epoch $k$ as following: 
\begin{enumerate}
\item Let $U_1 = 0$
\item If $d(\mathcal H_k)\ge 0$:
let $U_{k+1}=U_k$
\item If $d(\mathcal H_k) <0$:
\begin{enumerate}
\item If the verification passed, let $U_{k+1}=U_k+1$
\item If the verification failed, let $U_{k+1} = U_k-1$
\end{enumerate}
\end{enumerate}
\end{definition}

\begin{lemma}

For all epoch $k$, we have
$$
\Pr(U_{k+1} = U_k + 1) \le 1-p
$$
\end{lemma}

\begin{proof}
Based on how $U_{k+1}$ is defined, we consider the case where $d(\mathcal H_k) <0$. It suffices to show that the probability that the algorithm zooms out is at least $p$. By definition, we have
\begin{equation}
\min_{h\notin \mathcal H_k} \epsilon_{\mathbb P_{\mathcal D_{r(k)}}}(h) < \min_{h\in \mathcal H_k} \epsilon_{\mathbb P_{\mathcal D_{r(k)}}}(h)
\end{equation}
Therefore, by applying Lemma~\ref{lemma:vc}, with probability at least $p$, we have
\begin{equation}
\min_{h\notin\mathcal H_k} \epsilon_{\mathcal Z_k'} (h) \le \min_{h\notin \mathcal H_k} \epsilon_{\mathbb P_{\mathcal D_{r(k)}}}(h) + \Delta(M,1-\sqrt p)
\end{equation}
and
\begin{equation}
\min_{h\in\mathcal H_k} \epsilon_{\mathcal Z_k'} (h) \ge \min_{h\in \mathcal H_k} \epsilon_{\mathbb P_{\mathcal D_{r(k)}}}(h) - \Delta(M,1-\sqrt p).
\end{equation}
Hence,
\begin{equation}
\min_{h\notin\mathcal H_k} \epsilon_{\mathcal Z_k'} (h) - \min_{h\in\mathcal H_k} \epsilon_{\mathcal Z_k'} (h) \le  2\Delta(M,1-\sqrt p).
\end{equation}
Therefore, the probability that the algorithm zooms out is at least $p$.
\end{proof}

\begin{lemma}
\label{lemma:8}
Let $R_k$ be the regret at epoch $k$ and $Q_k$. Then, there exists $c_6>0$ such that
\begin{equation}
\mathbb E[R_k|U_k] \le c_6^{U_K}\cdot 2M
\end{equation}
\end{lemma}
\begin{proof}
First, note that when $h^*\in\mathcal H_k$, there is no regret at epoch $k$. When $h^*\notin H_k$, and let $r'<k$ be the last time such that $h^*\in H_{r'}$, then, at epoch $k$ we have
\begin{equation}
\Pr(r_t=1) \le \Pr(X_t\notin \mathcal D_k)={\phi(\mathcal D_{r'})}- \phi(\mathcal D_k).
\end{equation}
By definition 2, number of times the algorithm zooms in from epoch $r'$ to epoch $k$ is $U_k$, by lemma 4 we have 
\begin{equation}
\Pr(r_t=1|U_k) \le  {\phi(\mathcal D_{r'})}- \phi(\mathcal D_k)  \le \frac{1-c_5^{U_K}}{c_5^{U_k}} {\phi(\mathcal D_{r'})} \le (1/c_5)^{U_k}  \phi(\mathcal D_k).
\end{equation}
Note that the number of queried labels at epoch $k$ is at most $2M$. Therefore,
\begin{equation}
\mathbb E[R_k|U_k] =(1/c_5)^{U_k}\cdot 2M
\end{equation}
as desired for $c_6 = 1/c_5$.
\end{proof}

\begin{lemma}

If $\displaystyle p>\frac 12- \frac 1 {4c}+\sqrt{\frac{1-c}{4c}}$ where $c=32 \theta c_0 m^{-\frac 1 2}$, there exists $c_7>0$ such that
$$
\mathbb E[R_k] \le c_7 M.
$$
\end{lemma}

\begin{proof}
By tower property and similarly as Lemma~\ref{lemma:tower}, we have,
\begin{equation}
\mathbb E[R_k] = \mathbb E[\mathbb E[R_k|U_k] ] = \mathbb E[c_6^{U_k}].
\end{equation}
Since $U_k$ is a random walk process with negative bias at least $p$. By interchanging the sum order and using the Moment Generating Function for Binomial random variable, we have
\begin{equation}
\mathbb E[c_2^{U_k}] \le = \left(\frac{1-p^2}{p^2}\right)^k
\end{equation}
as desired.
\end{proof}

By Theorem 2 we have $\mathbb E[Q(T)]\le  C md \log T$. By Lemma~\ref{lemma:8} we have 
\begin{equation}
\mathbb E[R(T)] \le \mathbb E\left[\sum_{k=1}^{k_t}R_k\right] \le c_7\mathbb E\left[\sum_{k=1}^{k_t}Q_k\right] \le c_7 \mathbb E[Q(T)] \le c_7 Cmd\log T
\end{equation}
 as desired.

\bibliography{OLA}
\bibliographystyle{IEEEtran}

\end{document}

%% file: macros.tex
\newcommand{\cD}{{\cal D}}

\newcommand{\cH}{{\cal H}}

\newcommand{\cX}{{\cal X}}

\newtheorem{theorem}{Theorem}
\newtheorem{lemma}{Lemma}

\newtheorem{definition}{Definition}
\newtheorem{corollary}{Corollary}

\newcommand{\beq}{\begin{equation}}
\newcommand{\eeq}{\end{equation}}
\newcommand{\bea}{\begin{array}}
\newcommand{\ena}{\end{array}}
\newcommand{\bds}{\begin {itemize}}
\newcommand{\eds}{\end {itemize}}
\newcommand{\bdf}{\begin{definition}}
\newcommand{\blm}{\begin{lemma}}
\newcommand{\edf}{\end{definition}}
\newcommand{\elm}{\end{lemma}}
\newcommand{\bthm}{\begin{theorem}}
\newcommand{\ethm}{\end{theorem}}
\newcommand{\bprp}{\begin{prop}}
\newcommand{\eprp}{\end{prop}}
\newcommand{\bcl}{\begin{claim}}
\newcommand{\ecl}{\end{claim}}
\newcommand{\bcr}{\begin{coro}}
\newcommand{\ecr}{\end{coro}}
\newcommand{\bquest}{\begin{question}}
\newcommand{\equest}{\end{question}}

\newcommand{\larrow}{{\larrow}}

\def\urltilda{\kern -.15em\lower .7ex\hbox{\~{}}\kern .04em}